\documentclass[letterpaper]{article} 
\usepackage{aaai23}  
\usepackage{times}  
\usepackage{helvet}  
\usepackage{courier}  
\usepackage[hyphens]{url}  
\usepackage{graphicx} 
\usepackage{natbib}  
\usepackage{caption} 
\usepackage{algorithm}
\usepackage{algorithmic}

\usepackage{amsmath}
\usepackage{amssymb}
\usepackage{mathtools}
\usepackage{amsthm}
\usepackage{bm}
\usepackage{enumerate}
\usepackage{booktabs}

\DeclareMathOperator{\Sym}{Sym}
\DeclareMathOperator{\erf}{erf}
\DeclareMathOperator{\rank}{rank}

\theoremstyle{plain}

\newtheorem{proposition}{Proposition}

\newtheorem{corollary}{Corollary}
\theoremstyle{definition}

\newtheorem{example}{Example}
\theoremstyle{remark}
\newtheorem{remark}{Remark}

\usepackage{newfloat}
\usepackage{listings}
\DeclareCaptionStyle{ruled}{labelfont=normalfont,labelsep=colon,strut=off} 
\floatstyle{ruled}
\newfloat{listing}{tb}{lst}{}
\floatname{listing}{Listing}
\title{Knowledge Graph Embedding by Normalizing Flows}
\author {
    Changyi Xiao\textsuperscript{\rm 1},
    Xiangnan He\textsuperscript{\rm 1}\thanks{Corresponding author.},
    Yixin Cao\textsuperscript{\rm 2}
}
\affiliations {
    \textsuperscript{\rm 1}School of Data Science, University of Science and Technology of China\\
    \textsuperscript{\rm 2}School of Computing and Information System, Singapore Management University\\
    changyi@mail.ustc.edu.cn,
    \{xiangnanhe, caoyixin2011\}@gmail.com
}

\usepackage{bibentry}

\begin{document}

\maketitle

\begin{abstract}
A key to knowledge graph embedding (KGE) is to choose a proper representation space, e.g., point-wise Euclidean space and complex vector space. In this paper, we propose a unified perspective of embedding and introduce uncertainty into KGE from the view of group theory. Our model can incorporate existing models (i.e., generality), ensure the computation is tractable (i.e., efficiency) and enjoy the expressive power of complex random variables (i.e., expressiveness). The core idea is that we embed entities/relations as elements of a symmetric group, i.e., permutations of a set. Permutations of different sets can reflect different properties of embedding. And the group operation of symmetric groups is easy to compute. In specific, we show that the embedding of many existing models, point vectors, can be seen as elements of a symmetric group. To reflect uncertainty, we first embed entities/relations as permutations of a set of random variables. A permutation can transform a simple random variable into a complex random variable for greater expressiveness, called a normalizing flow. We then define scoring functions by measuring the similarity of two normalizing flows, namely NFE. We construct several instantiating models and prove that they are able to learn logical rules. Experimental results demonstrate the effectiveness of introducing uncertainty and our model. The code is available at \url{https://github.com/changyi7231/NFE}.
\end{abstract}

\section{Introduction}
A knowledge graph (KG) contains a large number of triplets with form $(h, r, t)$, where $h$ is a head entity, $r$ is a relation and $t$ is a tail entity. Existing KGs often suffer from the incompleteness problem. To complement the KGs, knowledge graph embedding (KGE) models map entities and relations into low-dimensional distributed representations and define scoring functions to measure the likelihood of the triplets.

A key to KGE is to choose a proper representation space such that the embedding can make good use of the properties of the space \cite{ji2021survey}. For example, point-wise Euclidean space is widely applied for representing entities/relations due to its simplicity yet effectiveness\cite{yang2014embedding}. Complex vector space is then proposed to learn sophisticated logical rules \cite{sun2018rotate,toutanova2015representing}. Moreover, the usage of quaternion space brings more degree of freedom of rotation than complex space \cite{zhang2019quaternion}. Besides, hyperbolic space is used to capture hierarchical logical rules \cite{chami2020low}. Motivated by the group theory, non-Abelian group is exploited to find the hidden group algebraic structure of relations \cite{yang2020nage}. However, the above works hardly consider the uncertainty information of embedding, which limit the expressiveness.

To reflect uncertainty, a common method is to embed entities/relations as random variables \cite{he2015learning}. However, it has two drawbacks. First, it is difficult to parameterize a complex random variable directly due to the difficulty of computing the partition function \cite{goodfellow2016deep}. Second, it is intractable to compute the probability density function (PDF) of the sum of two random variables due to the difficulty of computing the convolution of two PDFs \cite{bertsekas2008introduction}.

In this paper, we propose a unified perspective of embedding and introduce uncertainty into KGE from the view of group theory. Our model makes the embedding more general and ensures the computation is tractable. We embed entities/relations as elements of a symmetric group, i.e., permutations of a set. (A permutation of a set $X$ is an invertible function from $X$ to $X$.) On one hand, since point vectors can be seen as elements of a symmetric group, our proposed embedding is a generalization of point vectors. Moreover, Cayley's theorem \cite{hungerford2012abstract}, every group is isomorphic to a subgroup of a symmetric group, also shows the generality of our proposed embedding. On the other hand, the group operation of symmetric groups, the composition of functions, is easy to compute. We can embed entities/relations as permutations of different sets to reflect different properties of embedding. For example, if the set is a vector space (or a random variable space), then the permutations of that set reflect the certainty (or uncertainty) property.

In specific, to reflect uncertainty, we embed entities/relations as permutations of a set of random variables. The permutation or invertible function can transform a simple random variable into a complex random variable for greater expressiveness, which is called a normalizing flow \cite{papamakarios2021normalizing,kobyzev2020normalizing}. We can easily parameterize a complex invertible function (instead of a complex random variable), and compute the PDF of a random variable obtained by an invertible function acting on a simple random variable \cite{papamakarios2021normalizing,kobyzev2020normalizing}. Our method can be seen as a generalized reparameterization method \cite{kingma2013auto,rezende2014stochastic}. We then define the scoring function by measuring the similarity of two normalizing flows, namely Normalizing Flows Embedding (NFE). Finally, we instantiate several NFE models by choosing different invertible functions. We further prove that NFE is able to learn logical rules.

The main contributions of this paper are listed below:
\begin{enumerate}[1.]
\item We propose a unified perspective of embedding from the view of group theory, which offers a rigorous theoretical understanding of KGE.
\item We proposed NFE to introduce uncertainty into KGE, which enjoys generality, efficiency and expressiveness.
\item Experimental results demonstrate the effectiveness of introducing uncertainty into KGE and our model.
\end{enumerate}

\section{Background}
In this section, we introduce the related background of our model, knowledge graph embedding, normalizing flows and group theory.

\subsection{Knowledge Graph Embedding}
\label{section:2.1}
Let $\mathcal{E}$ and $\mathcal{R}$ denote the sets of entities and relations, respectively. A KG contains a set of triplets $\mathcal{F}=\{(h,r,t)\}\subset \mathcal{E}\times \mathcal{R}\times \mathcal{E}$. KGE aims at learning the distributed representations or embeddings for entities and relations. It defines a scoring function $f(h,r,t)$ to measure the likelihood of a triplet $(h,r,t)$ based on the embeddings $(\bm{h}, \bm{r}, \bm{t})$. Existing KGE models include translation-based models, multiplicative models and so on \cite{zhang2021neural}.

Translation-based models learn embeddings by translating a head entity to a tail entity through a relation. TransE \cite{bordes2013translating}, a representative model of translation-based models, defines the scoring function as the negative distance between $\bm{h}+\bm{r}$ and $\bm{t}$, i.e., $$f(\bm{h}, \bm{r},\bm{t})=-\|\bm{h}+\bm{r}-\bm{t}\|$$ where $\bm{h}, \bm{r}, \bm{t}\in \mathbb{R}^{n}$ and $\|\cdot\|$ is a norm of a vector.

Multiplicative models measure the likelihood of a triplet by product-based similarity of entities and relations. DistMult \cite{yang2014embedding}, a representative model of multiplicative models, defines the scoring function as the inner product of $\bm{h}$, $\bm{r}$ and $\bm{t}$, i.e., $$f(\bm{h}, \bm{r}, \bm{t})=\langle \bm{h}, \bm{r}, \bm{t}\rangle:=\sum_{i=1}^{n}\bm{h}_{i}\bm{r}_{i}\bm{t}_{i}$$
where $\bm{h}, \bm{r}, \bm{t}\in \mathbb{R}^{n}$ and $\langle \cdot,\cdot,\cdot \rangle$ is the inner product of three vectors.

\subsection{Normalizing Flows}
\label{section:2.2}
A normalizing flow is a sequence of invertible and differentiable functions, which transforms a simple probability distribution ($e.g.$, a standard normal) into a more complex distribution \cite{papamakarios2021normalizing,kobyzev2020normalizing}. Let $\bm{Z}$ be a random variable in $\mathbb{R}^n$ with a known and tractable PDF $p_{\bm{Z}}(\bm{z})$ and $\bm{X}=g(\bm{Z})$ be an invertible function which transforms $\bm{Z}$ into $\bm{X}$. The PDF of the random variable $\bm{X}$ follows
\begin{align}
p_{\bm{X}}(\bm{x})=p_{\bm{Z}}(g^{-1}(\bm{x}))|\det(\frac{\partial g^{-1}(\bm{x})}{\partial \bm{x}})|
\label{equation:1}
\end{align}
where $g^{-1}$ is the inverse of $g$ and $\det(\frac{\partial g^{-1}(\bm{x})}{\partial \bm{x}})$ is the determinant of the Jacobian of $g^{-1}$ evaluated at $\bm{x}$. We next show an example of normalizing flows.
\begin{example}[Pushing uniform to normal]
\label{example:1}
Let $z\sim U[0,1]$ be uniformly distributed and $x\sim \mathcal{N}(\mu,\sigma^{2})$ be normally distributed. The invertible function
$$x=S(z)=\mu+\sqrt{2}\sigma\cdot \erf^{-1}(2z-1)$$
pushes $z$ into $x$, where $\erf(z)=\frac{2}{\sqrt{\pi}}\int_{0}^{z}e^{-s^2}\mathrm{d}s$ is the error function.
\end{example}

\subsection{Group Theory}
\label{section:2.3}
A group is a set $G$ together with a group operation on $G$, here denoted $\cdot$, that combines any two elements $a$ and $b$ to form an element of $G$, denoted $a\cdot b$, such that the following three requirements are satisfied:
\begin{enumerate}
    \item \textbf{Associativity} $\forall a,b,c\in G, (a\cdot b)\cdot c=a\cdot (b\cdot c)$.
    \item \textbf{Identity element} $\exists e\in G$ such that $\forall a\in G, e\cdot a=a$ and $a\cdot e=a$.
    \item \textbf{Inverse element} $\forall a\in G,\exists b\in G$ such that $a\cdot b=e$ and $b\cdot a=e$.
\end{enumerate}
A permutation group is a group $G$ whose elements are permutations of a given set $X$ and whose group operation is the composition of functions in $G$, where permutations are invertible functions from the set $X$ to itself. The group of all permutations of a set $X$ is the symmetric group of $X$, denoted as $\Sym(X)$. The term permutation group thus means a subgroup of the symmetric group. By Cayley's theorem \cite{hungerford2012abstract}, every group is isomorphic to a permutation group, which characterizes group structure as the structure of a set of invertible functions.

\section{Normalizing Flows Embedding}
In this section, we first propose a unified perspective of embedding and extend to the general form of normalizing flows embedding (NFE). We then define the concrete forms of NFE and prove that they are able to learn logical rules. Finally, we show the loss function.

\subsection{Unified Perspective of Embedding}
\label{section:3.1}
Most KGE models embed entities/relations as point vectors in a low-dimensional space. We propose a unified perspective of embedding from the view of symmetric groups and show that point vectors can be seen as elements of a symmetric group. We take the representative KGE models, TransE and DistMult, as examples to illustrate. We show the results for other models in Appendix B.
\paragraph{TransE}
Let $G=X=\mathbb{R}^{n}$, TransE embeds every entity/relation as a vector in $G$. We show that every vector can correspond to a permutation of $X$, i.e., an element of the symmetric group $\Sym(X)$. We define a map $\alpha$ from $G$ to $\Sym(X)$:
\begin{align*}
    &\alpha: G\rightarrow \Sym(X)\\
    &\alpha(\bm{g})=f_{\bm{g}}
\end{align*}
where $\bm{g}\in G$ and the definition of $f_{\bm{g}}$ is as follows:
\begin{align*}
    &f_{\bm{g}}: X\rightarrow X\\
    &f_{\bm{g}}(\bm{x})=\bm{g}+\bm{x}
\end{align*}
where $\bm{x}\in X$. The image of $\alpha$ is a permutation group, i.e., a subgroup of $\Sym(X)$. For every triplet $(\bm{h}, \bm{r}, \bm{t})$, we have:
$$f_{\bm{h}}(\bm{x})=\bm{h}+\bm{x}, f_{\bm{r}}(\bm{x})=\bm{r}+\bm{x}, f_{\bm{t}}(\bm{x})=\bm{t}+\bm{x}$$
then TransE can be represented as:
\begin{align*}
f(\bm{h}, \bm{r}, \bm{t})&=-\|\bm{h}+\bm{r}-\bm{t}\| \\
                         &=-\|\bm{r}+(\bm{h}+\bm{0})-(\bm{t}+\bm{0})\| \\
                         &= D(f_{\bm{r}}(f_{\bm{h}}(\bm{x_{0}})), f_{\bm{t}}(\bm{x_{0}})) \\
                         &= D(f_{\bm{r}}\circ f_{\bm{h}}(\bm{x_{0}}), f_{\bm{t}}(\bm{x_{0}}))
\end{align*}
where $\bm{x}_{0}=\bm{0}\in X$, $\circ$ denotes functional composition and $D(\bm{a},\bm{b})=-\|\bm{a}-\bm{b}\|, \bm{a}, \bm{b}\in \mathbb{R}^{n}$. $D(\cdot,\cdot)$ is a similarity metric, which outputs a real value.
\paragraph{DistMult}
Let $G=(\mathbb{R}\backslash\{0\})^n$ and $X=\mathbb{R}^{n}$, DistMult embeds every entity/relation as a vector in $G$. We can similarly define a map $\alpha$ from $G$ to $\Sym(X)$:
\begin{align*}
    &\alpha: G\rightarrow \Sym(X)\\
    &\alpha(\bm{g})=f_{\bm{g}}
\end{align*}
where $\bm{g}\in G$ and the definition of $f_{\bm{g}}$ is as follows:
\begin{align*}
    &f_{\bm{g}}: X\rightarrow X\\
    &f_{\bm{g}}(\bm{x})=\bm{g}\odot\bm{x}
\end{align*}
where $\bm{x}\in X$ and $\odot$ is Hadamard product. The image of $\alpha$ is another subgroup of $\Sym(X)$.
For every triplet $(\bm{h}, \bm{r}, \bm{t})$, we have:
$$f_{\bm{h}}(\bm{x})=\bm{h}\odot\bm{x}, f_{\bm{r}}(\bm{x})=\bm{r}\odot\bm{x}, f_{\bm{t}}(\bm{x})=\bm{t}\odot\bm{x}$$
then DistMult can be represented as:
\begin{align*}
f(\bm{h}, \bm{r}, \bm{t})&=\langle \bm{h}, \bm{r}, \bm{t} \rangle \\
                         &=\langle \bm{r}\odot(\bm{h}\odot \bm{1}), \bm{t}\odot \bm{1} \rangle \\
                         &= D(f_{\bm{r}}(f_{\bm{h}}(\bm{x_{0}})), f_{\bm{t}}(\bm{x_{0}})) \\
                         &= D(f_{\bm{r}}\circ f_{\bm{h}}(\bm{x_{0}}), f_{\bm{t}}(\bm{x_{0}}))
\end{align*}
where $\bm{x}_{0}=\bm{1}\in X$ and  $D(\bm{a},\bm{b})=\bm{a}^{T}\bm{b}, \bm{a}, \bm{b}\in \mathbb{R}^{n}$.
\paragraph{Unified Representation}
Therefore, we get a unified representation of TransE and DistMult from the view of symmetric groups:
\begin{align}
  f(\bm{h}, \bm{r}, \bm{t})= D(f_{\bm{r}}\circ f_{\bm{h}}(\bm{x_{0}}), f_{\bm{t}}(\bm{x_{0}}))
  \label{equation:2}
\end{align}
The unified representation is defined by measuring the similarity of $f_{\bm{r}}\circ f_{\bm{h}}$ evaluated at $\bm{x_{0}}$ and $f_{\bm{t}}$ evaluated at $\bm{x_{0}}$. It is composed of three parts, the initial object $\bm{x}_{0}\in X$, the permutations $\{f_{\bm{h}}, f_{\bm{r}}, f_{\bm{t}}\}$ and the similarity metric $D(\cdot,\cdot)$. Thus, we can define scoring functions in terms of permutations of a set.

TransE and DistMult embed entities/relations into different subgroup of the same symmetric group $\Sym(\mathbb{R}^n)$ to get different scoring functions. Since every point vector corresponds to an element of $\Sym(\mathbb{R}^n)$, the elements of $\Sym(\mathbb{R}^n)$ can reflect the certainty property of point vectors. Thus, we can reflect different properties of embedding and define different scoring functions by choosing different symmetric groups. Next, we introduce the uncertainty of embedding by choosing a suitable symmetric group.

\subsection{General Form of NFE}
To introduce uncertainty, we let $X$ be the set of random variables on $\mathbb{R}^n$ and embed entities/relations as elements of $\Sym(X)$. For every triplet $(\bm{h}, \bm{r}, \bm{t})$, we have the corresponding permutations or invertible functions $\{f_{\bm{h}}, f_{\bm{r}}, f_{\bm{t}}\}$. Our NFE is defined in the form of Eq.(\ref{equation:2}), where $\bm{x}_{0} \in X$ is a random variable and $D(\cdot,\cdot)$ is a similarity metric between two random variables or two PDFs. $f_{\bm{r}}\circ f_{\bm{h}}(\bm{x_{0}})$ or $f_{\bm{t}}(\bm{x_{0}})$ is a normalizing flow, which transforms a random variable $\bm{x}_{0}$ into another random variable, this reflects uncertainty information of  $\{f_{\bm{h}}, f_{\bm{r}}, f_{\bm{t}}\}$. Thus, NFE is to measure the similarity of two normalizing flows. We can easily compute the PDFs of $f_{\bm{r}}\circ f_{\bm{h}}(\bm{x_{0}})$ and $f_{\bm{t}}(\bm{x_{0}})$ by Eq.(\ref{equation:1}), then the general form of NFE is represented as
\begin{align}
f(\bm{h}, \bm{r}, \bm{t})= D(&f_{\bm{r}}\circ f_{\bm{h}}(\bm{x_{0}}), f_{\bm{t}}(\bm{x_{0}}))\notag\\
=D(&p_{\bm{x}_{0}}(f_{\bm{h}}^{-1}\circ f_{\bm{r}}^{-1}(\bm{x}))|\det(\frac{\partial f_{\bm{h}}^{-1}\circ f_{\bm{r}}^{-1}(\bm{x})}{\partial \bm{x}})|,\notag\\
&p_{\bm{x}_{0}}(f_{\bm{t}}^{-1}(\bm{x}))|\det(\frac{\partial f_{\bm{t}}^{-1}(\bm{x})}{\partial \bm{x}})|)
\label{equation:3}
\end{align}
In next section, we define concrete $\bm{x}_0$, $\{f_{\bm{h}}, f_{\bm{r}}, f_{\bm{t}}\}$ and $D(\cdot,\cdot)$ to get the concrete forms of NFE.

\paragraph{Comparison of KG2E and NFE}
KG2E \cite{he2015learning} embeds entities/relations as random variables and defines the scoring function as $f(\bm{h}, \bm{r}, \bm{t})=D(\bm{h}-\bm{t},\bm{r})$, where $\bm{h},\bm{r}$ and $\bm{t}$ are random variables with normal distributions. However, it has two drawbacks to represent entities/relations as random variables to model uncertainty. First, it is difficult to parameterize a complex random variable directly due to the difficulty of computing the partition function \cite{goodfellow2016deep}. We get a complex random variable by using a complex invertible function to act on a simple random variable. It is easy to parameterize a complex invertible function and compute the PDF of a random variable obtained by an invertible function acting on a simple random variable by Eq.(\ref{equation:1}) \cite{papamakarios2021normalizing,kobyzev2020normalizing}. Thus, our method can be seen as a generalized reparameterization method \cite{kingma2013auto,rezende2014stochastic}. Second, KG2E involves computing the PDF of the sum/difference of two random variables, i.e., the PDF of $\bm{h}-\bm{t}$, which is intractable in most cases due to the difficulty of computing the convolution of two PDFs \cite{bertsekas2008introduction}. For example, if we embed a head entity $\bm{h}$ and a relation $\bm{r}$ as random variables with Beta distribution, then it is hard to compute the PDF of $\bm{h}-\bm{t}$. In contrast, Eq.(\ref{equation:3}) requires computing the PDF of $f_{\bm{r}}\circ f_{\bm{h}}(\bm{x}_0)$, which is still a normalizing flow because a composition of invertible functions remains invertible. It is easy to compute the composition of $f_{\bm{h}}$ and $f_{\bm{r}}$ and the PDF of $f_{\bm{r}}\circ f_{\bm{h}}(\bm{x}_0)$ by Eq.(\ref{equation:1}). In conclusion, NFE is more general and computationally simple.

\subsection{Concrete Form of NFE}
\label{section:3.2}
Based on the general form of NFE, Eq.(\ref{equation:3}), we define the concrete initial random variables $\bm{x}_0$, invertible functions $\{f_{\bm{h}}, f_{\bm{r}}, f_{\bm{t}}\}$ and similarity metrics $D(\cdot,\cdot)$ to get the concrete forms of NFE.

For the initial random variable $\bm{x}_0$, we can choose it to be a simple random variable for the convenience of computing Eq.(\ref{equation:3}), such as a random variable with uniform distribution $U[-\sqrt{3},\sqrt{3}]^n$ or a random variable with standard normal distribution $\mathcal{N}(\bm{0},\bm{I})$.

The invertible functions $f_{\bm{g}}$ can be chosen as $f_{\bm{g}}(\bm{x})=\bm{g}+\bm{x}$ as in TransE or $f_{\bm{g}}(\bm{x})=\bm{g}\odot\bm{x}$ as in DistMult. We use the composition of the two functions:
\begin{align}
f_{\bm{g}}(\bm{x})=\bm{g}_{\sigma}\odot\bm{x}+\bm{g}_{\mu}
\label{equation:4}
\end{align}
where $\bm{g}_{\sigma}\in \mathbb{R}^n$, the entries of $\bm{g}_{\sigma}$ are non-zero, and $\bm{g}_{\mu}\in \mathbb{R}^n$. Thus, for every triplet $(\bm{h}, \bm{r}, \bm{t})$, we have that
\begin{align}
&f_{\bm{h}}(\bm{x})=\bm{h}_{\sigma}\odot\bm{x}+\bm{h}_{\mu},f_{\bm{r}}(\bm{x})=\bm{r}_{\sigma}\odot\bm{x}+\bm{r}_{\mu}\notag\\
&f_{\bm{r}}\circ f_{\bm{h}}(\bm{x})=\bm{r}_{\sigma}\odot\bm{h}_{\sigma}\odot\bm{x}+\bm{r}_{\sigma}\odot\bm{h}_{\mu}+\bm{r}_{\mu}\notag\\
&f_{\bm{t}}(\bm{x})=\bm{t}_{\sigma}\odot\bm{x}+\bm{t}_{\mu}
\label{equation:5}
\end{align}
We denote the PDF of $f_{\bm{r}}\circ f_{\bm{h}}(\bm{x}_0)$ as $p_{\bm{r}\bm{h}}(\bm{x}_0)$ and the PDF of $f_{\bm{t}}(\bm{x}_0)$ as $q_{\bm{t}}(\bm{x}_0)$. If $\bm{x}_{0}\sim \mathcal{N}(\bm{0},\bm{I})$ and $f_{\bm{g}}$ is a linear (affine) function $f_{\bm{g}}(\bm{x})=\bm{g}_{\sigma}\odot\bm{x}+\bm{g}_{\mu}$, then $p_{\bm{r}\bm{h}}(\bm{x}_0)$ and $q_{\bm{t}}(\bm{x}_0)$ are still normal distributions. A more expressive choice of $f_{\bm{g}}$ than Eq.(\ref{equation:4}) can be piecewise linear functions:
\begin{align}
   f_{\bm{g}}(\bm{x})_i=
   \begin{cases}
   \bm{g}_{\sigma_{1}i}\odot\bm{x}_{i}+\bm{g}_{\mu i} & \text{if }\bm{x}_{i}\leq 0\\
   \bm{g}_{\sigma_{2}i}\odot\bm{x}_{i}+\bm{g}_{\mu i} & \text{if }\bm{x}_{i}>0
   \end{cases}
   \label{equation:6}
\end{align}
where $\bm{g}_{\sigma_{1}},\bm{g}_{\sigma_{2}},\bm{g}_{\mu}\in \mathbb{R}^n$. For simplicity, we denote Eq.(\ref{equation:6}) as
\begin{align*}
   f_{\bm{g}}(\bm{x})=
   \begin{cases}
   \bm{g}_{\sigma_{1}}\odot\bm{x}+\bm{g}_{\mu} & \text{if }\bm{x}\leq \bm{0}\\
   \bm{g}_{\sigma_{2}}\odot\bm{x}+\bm{g}_{\mu} & \text{if }\bm{x}>\bm{0}
   \end{cases}
\end{align*}
To ensure Eq.(\ref{equation:6}) is invertible, we need to constraint $\bm{g}_{\sigma_{1}}\odot \bm{g}_{\sigma_{2}}> \bm{0}$. Since the composition of two piecewise linear functions with two pieces is a piecewise linear function with three pieces, we still implement $f_{\bm{r}}$ as a linear function. This ensures that for any $f_{\bm{h}}(\bm{x})$ and $f_{\bm{r}}(\bm{x})$, there exists a $f_{\bm{t}}(\bm{x})$ such that $f_{\bm{t}}(\bm{x})=f_{\bm{r}}\circ f_{\bm{h}}(\bm{x})$ for any $\bm{x}$. Thus, for every triplet $(\bm{h}, \bm{r}, \bm{t})$, we have that
\begin{align}
    &f_{\bm{h}}(\bm{x})=
   \begin{cases}
   \bm{h}_{\sigma_{1}}\odot\bm{x}+\bm{h}_{\mu} & \text{if }\bm{x}\leq \bm{0}\\
   \bm{h}_{\sigma_{2}}\odot\bm{x}+\bm{h}_{\mu} & \text{if }\bm{x}>\bm{0}
   \end{cases},\notag\\
   &f_{\bm{r}}(\bm{x})=\bm{r}_{\sigma}\odot\bm{x}_{i}+\bm{r}_{\mu},\notag\\
   &f_{\bm{r}}\circ f_{\bm{h}}(\bm{x})=
   \begin{cases}
   \bm{r}_{\sigma}\odot(\bm{h}_{\sigma_{1}}\odot\bm{x}+\odot\bm{h}_{\mu})+\bm{r}_{\mu} & \text{if }\bm{x}\leq \bm{0}\\
   \bm{r}_{\sigma}\odot(\bm{h}_{\sigma_{2}}\odot\bm{x}+\odot\bm{h}_{\mu})+\bm{r}_{\mu} & \text{if }\bm{x}>\bm{0}
   \end{cases}, \notag\\
   &f_{\bm{t}}(\bm{x})=
   \begin{cases}
   \bm{t}_{\sigma_{1}}\odot\bm{x}+\bm{t}_{\mu} & \text{if }\bm{x}\leq \bm{0}\\
   \bm{t}_{\sigma_{2}}\odot\bm{x}+\bm{t}_{\mu} & \text{if }\bm{x}>\bm{0}
   \end{cases}
   \label{equation:7}
\end{align}

The similarity metric $D(\cdot,\cdot)$ can be chosen as the negative Kullback–Leibler (KL) divergence between the PDFs, $p(\bm{x})$ and $q(\bm{x})$,
\begin{align}
D(p,q)=-KL(p,q)=-\int p(\bm{x})\log\frac{p(\bm{x})}{q(\bm{x})}\mathrm{d}\bm{x}
\label{equation:8}
\end{align}
or negative Wasserstein distance between the PDFs \cite{peyre2019computational}, $p(\bm{x})$ and $q(\bm{y})$,
\begin{align}
D(p,q)&=-W(p,q)\notag\\
&=-\inf_{\gamma\in \Gamma(p,q)}\iint \|\bm{x}-\bm{y}\|_2^2\gamma(\bm{x},\bm{y})\mathrm{d}\bm{x}\mathrm{d}\bm{y}
\label{equation:9}
\end{align}
where $\Gamma$ denotes the set of all joint distributions of $(\bm{x},\bm{y})$. Wasserstein distance is still valid if the support sets of $p(\bm{x})$ and $q(\bm{y})$ are not overlapped, while the KL divergence is not valid. For example, $W(p(\bm{x}),q(\bm{y}))=4$ and $KL(p(\bm{x}),q(\bm{y}))=\infty$ if $\bm{x}\sim U[0,1]$ and $\bm{y}\sim U[2,3]$. Therefore, we use Wasserstein distance instead of KL divergence. However, Wasserstein distance is difficult to compute in most cases \cite{peyre2019computational}. It has a tractable solution when $n=1$:$$W(p,q)=\int_0^1(F^{-1}(z)-G^{-1}(z))^2\mathrm{d}z$$
where $F^{-1}(z)$ and $G^{-1}(z)$ are the inverse cumulative distribution function of $p(\bm{x})$ and $q(\bm{y})$, respectively. Our idea is to decompose the $n$-dimensional Wasserstein distance into a sum of 1-dimensional Wasserstein distances. We have the following proposition to realize it.
\begin{proposition}
\label{proposition:1}
For two PDFs, $p(\bm{x})$ and $q(\bm{y})$, $W(p,q)=\sum_{i=1}^{n}W(p_i,q_i)$ iff $p(\bm{x})$ and $q(\bm{y})$ share the same copula, where $p_i$ and $q_i$ are the marginal distributions of $\bm{x}_{i}$ and $\bm{y}_{i}$, respectively \cite{cuestaalbertos1993optimal}.
\end{proposition}
See Appendix C for the proofs. Thus, we can get the following corollary.
\begin{corollary}
\label{corollary:1}
For two PDFs, $p(\bm{x})$ and $q(\bm{y})$, $W(p,q)=\sum_{i=1}^{n}W(p(\bm{x}_i),q(\bm{y}_i))$ if $p(\bm{x})=\prod_{i=1}^{n}p(\bm{x}_i)$ and $q(\bm{y})=\prod_{i=1}^{n}q(\bm{y}_i)$.
\end{corollary}
We design proper scoring functions such that $p_{\bm{r}\bm{h}}(\bm{x}_0)$ and $q_{\bm{t}}(\bm{x}_0)$ satisfy the conditions in Corollary \ref{corollary:1}. In summary, we can define the concrete forms of NFE. We have the following propositions.
\begin{proposition}
\label{proposition:2}
If $\bm{x}_0\sim U[-\sqrt{3},\sqrt{3}]^n$ or $\bm{x}_0\sim \mathcal{N}(\bm{0},\bm{I})$, the invertible functions are Eq.(\ref{equation:5}) and similarity metric is Eq.(\ref{equation:9}), then the scoring function is
\begin{align}
f(\bm{h}, \bm{r}, \bm{t})=&-\|\bm{r}_{\sigma}\odot\bm{h}_{\mu}+\bm{r}_{\mu}-\bm{t}_{\mu}\|_2^2\notag\\
&-\||\bm{r}_{\sigma}\odot\bm{h}_{\sigma}|-|\bm{t}_{\sigma}|\|_2^2
\label{equation:10}
\end{align}
\end{proposition}
\begin{proposition}
If $\bm{x}_0\sim U[-\sqrt{3},\sqrt{3}]^n$, the invertible functions are  Eq.(\ref{equation:7}) and similarity metric is Eq.(\ref{equation:9}), then the scoring function is 
\begin{align}
&f(\bm{h}, \bm{r},\bm{t})=-\|\bm{r}_{\sigma}\odot\bm{h}_{\mu}+\bm{r}_{\mu}-\bm{t}_{\mu}\|_2^2\notag\\
&-\frac{1}{2}\||\bm{r}_{\sigma}\odot\bm{h}_{\sigma_{1}}|-|\bm{t}_{\sigma_{1}}|\|_2^2\notag\\
&-\frac{1}{2}\||\bm{r}_{\sigma}\odot\bm{h}_{\sigma_{2}}|-|\bm{t}_{\sigma_{2}}|\|_2^2-\sqrt{\frac{3}{4}}(\bm{r}_{\sigma}\odot\bm{h}_{\mu}+\bm{r}_{\mu}-\bm{t}_{\mu})^T\notag\\
&(|\bm{r}_{\sigma}\odot\bm{h}_{\sigma_{2}}|-|\bm{r}_{\sigma}\odot\bm{h}_{\sigma_{1}}|+|\bm{t}_{\sigma_{1}}|-|\bm{t}_{\sigma_{2}}|)
\label{equation:11}
\end{align}
\end{proposition}
\begin{proposition}
If $\bm{x}_0\sim \mathcal{N}(\bm{0},\bm{I})$, the invertible functions are  Eq.(\ref{equation:7}) and similarity metric is Eq.(\ref{equation:9}), then the scoring function is 
\begin{align}
&f(\bm{h}, \bm{r},\bm{t})=-\|\bm{r}_{\sigma}\odot\bm{h}_{\mu}+\bm{r}_{\mu}-\bm{t}_{\mu}\|_2^2\notag\\
&-\frac{1}{2}\||\bm{r}_{\sigma}\odot\bm{h}_{\sigma_{1}}|-|\bm{t}_{\sigma_{1}}|\|_2^2\notag\\
&-\frac{1}{2}\||\bm{r}_{\sigma}\odot\bm{h}_{\sigma_{2}}|-|\bm{t}_{\sigma_{2}}|\|_2^2-\sqrt{\frac{2}{\pi}}(\bm{r}_{\sigma}\odot\bm{h}_{\mu}+\bm{r}_{\mu}-\bm{t}_{\mu})^T\notag\\
&(|\bm{r}_{\sigma}\odot\bm{h}_{\sigma_{2}}|-|\bm{r}_{\sigma}\odot\bm{h}_{\sigma_{1}}|+|\bm{t}_{\sigma_{1}}|-|\bm{t}_{\sigma_{2}}|)
\label{equation:12}
\end{align}
\end{proposition}
The first term of Eq.(\ref{equation:10}) is to measure the difference of the mean of $p_{\bm{rh}}(\bm{x}_0)$ and $p_{\bm{t}}(\bm{x}_0)$, the second term is to measure the difference of the standard deviation of $p_{\bm{rh}}(\bm{x}_0)$ and $p_{\bm{t}}(\bm{x}_0)$. If $\bm{h}_{\sigma}=\bm{r}_{\sigma}=\bm{t}_{\sigma}=\bm{1}$, Eq.(\ref{equation:10}) reduces to $f(\bm{h}, \bm{r}, \bm{t})=-\|\bm{h}_{\mu}+\bm{r}_{\mu}-\bm{t}_{\mu}\|_2^2$, which is the same as TransE, and the second term of Eq.(\ref{equation:10}) is equal to 0, i.e., the standard deviations of $p_{\bm{r}\bm{h}}(\bm{x}_0)$ and $q_{\bm{t}}(\bm{x}_0)$ are the same.
If $\bm{h}_{\mu}=\bm{r}_{\mu}=\bm{t}_{\mu}=\bm{0}$, Eq.(\ref{equation:10}) reduces to $f(\bm{h}, \bm{r}, \bm{t})=-\||\bm{r}_{\sigma}\odot\bm{h}_{\sigma}|-|\bm{t}_{\sigma}|\|_2^2$. Eq.(\ref{equation:11}) and Eq.(\ref{equation:12}) are similar, the only difference is the coefficient of the fourth terms, one is $-\sqrt{\frac{3}{4}}$, the other is $-\sqrt{\frac{2}{\pi}}$. If $\bm{h}_{\sigma_{1}}=\bm{h}_{\sigma_{2}}$ and $\bm{t}_{\sigma_{1}}=\bm{t}_{\sigma_{2}}$, Eq.(\ref{equation:11}) or Eq.(\ref{equation:12}) reduces to Eq.(\ref{equation:10}).

If we choose $\bm{x}_0$ to be a random variable with Dirac delta distribution (i.e. a point vector), then Eq.(\ref{equation:3}) do not model uncertainty. Thus, NFE can be seen as a generalization of conventional models. We have the following proposition to illustrate this:
\begin{proposition}
\label{proposition:5}
Let $k>0$, if $\bm{x}_{k}\sim U[-\frac{\sqrt{3}}{k},\frac{\sqrt{3}}{k}]^n$ or $\bm{x}_{k}\sim \mathcal{N}(\bm{0},\frac{\bm{I}}{k^2})$, the invertible functions are  Eq.(\ref{equation:5}) and similarity metric is Eq.(\ref{equation:9}), denote the scoring function as $f_k(\bm{h}, \bm{r},\bm{t})$, then $\bm{x}_k$ tends to a random variable with Dirac distribution as $k$ tends to infinity and
\begin{align}
&\lim_{k\rightarrow \infty}f_k(\bm{h}, \bm{r},\bm{t})\notag\\
=&\lim_{k\rightarrow \infty}-\|\bm{r}_{\sigma}\odot\bm{h}_{\mu}+\bm{r}_{\mu}-\bm{t}_{\mu}\|_2^2-\frac{1}{k^2}\||\bm{r}_{\sigma}\odot\bm{h}_{\sigma}|-|\bm{t}_{\sigma}|\|_2^2\notag\\
=&-\|\bm{r}_{\sigma}\odot\bm{h}_{\mu}+\bm{r}_{\mu}-\bm{t}_{\mu}\|_2^2
\label{equation:13}
\end{align}
\end{proposition}
Proposition \ref{proposition:5} shows that $1/k^2$ in $f_k(\bm{h}, \bm{r},\bm{t})$ reflects the degree to which $f_k(\bm{h}, \bm{r},\bm{t})$ focuses on uncertainty. Higher value of $1/k^2$ indicates $f_k(\bm{h}, \bm{r},\bm{t})$ focusing more on uncertainty. If $1/k^2=0$, the second term of $f_k(\bm{h}, \bm{r},\bm{t})$ is dropped. Thus, the second term of $f_k(\bm{h}, \bm{r},\bm{t})$ or the second term of Eq.(\ref{equation:10}) is to model uncertainty. Eq.(\ref{equation:11}) or Eq.(\ref{equation:12}) has the similar result as Eq.(\ref{equation:10}). In conclusion, our proposed model is more general than conventional models.

\subsection{Logical Rules}
\label{section:3.3}
The inductive ability of a scoring function is reflected in its ability to learn logical rules \cite{zhang2021neural}. The symmetry, antisymmetry, inverse and composition rules are defined as follows:\\
\textbf{Symmetry Rules}: A relation $r$ is symmetric if $ \forall h,t, (h,r,t)\in \mathcal{F} \rightarrow (t,r,h)\in \mathcal{F}$.\\
\textbf{Antisymmetry Rules}: A relation $r$ is antisymmetric if $ \forall h,t, (h,r,t)\in \mathcal{F} \rightarrow (t,r,h)\notin \mathcal{F}$.\\
\textbf{Inverse Rules}: A relation $r_{1}$ is inverse to a relation $r_{2}$ if $ \forall h,t, (h,r_{1},t)\in \mathcal{F} \rightarrow (t,r_{2},h)\in \mathcal{F}$.\\
\textbf{Composition Rules}: A relation $r_{3}$ is the composition of a relation $r_{1}$ and a relation $r_{2}$ if $\forall h,\widetilde{t},t, (h,r_{1},\widetilde{t})\in \mathcal{F}\wedge (\widetilde{t},r_{2},t)\in \mathcal{F} \rightarrow (h,r_{3},t)\in \mathcal{F}$.\\
We have the following proposition about our proposed scoring functions and logical rules.
\begin{proposition}
\label{proposition:6}
Scoring functions Eq.(\ref{equation:10}), Eq.(\ref{equation:11}) and Eq.(\ref{equation:12}) can learn symmetry, antisymmetry, inverse and composition rules.
\end{proposition}

\subsection{Loss Function}
\label{section:3.4}
We use the same loss function, binary classification loss function with reciprocal learning, as in \citep{dettmers2018convolutional}. For every triplet $(h,r,t)$, our loss function is
\begin{align*}
    \ell(h,r,t)=\sum_{t^{'}\in\mathcal{E}}\log(1+\exp(-y_{t^{'}}(\gamma-f(h,r,t^{'}))))
\end{align*}
where $\gamma$ is a fixed margin and $y_{t^{'}}=1$ if $t^{'}=t$, otherwise $y_{t^{'}}=-1$.

\begin{figure*}[t]
\begin{center}
\centerline{\includegraphics[width=0.9\textwidth]{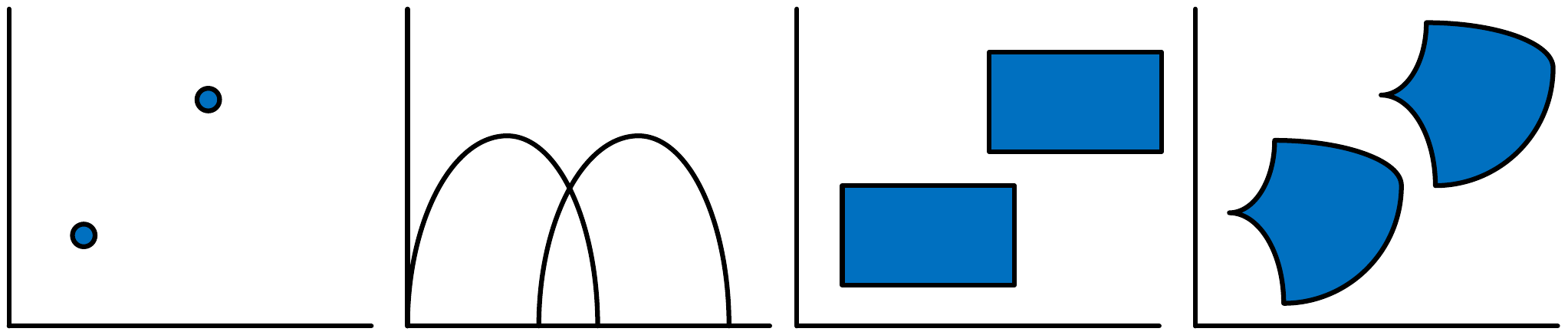}}
\caption{We define permutations $\alpha$ of different sets $X$ with the same form, $\alpha(\bm{g})(\bm{x})=\bm{g}+\bm{x}$, which translates one object to another. From left to right, the first corresponds to the translation of points, the second corresponds to the translation of random variables, the third corresponds to the translation of hyper-rectangles, and the fourth corresponds to the translation of manifolds.}
\label{figure:1}
\end{center}
\end{figure*}

\section{Discussion}
\label{section:6}
\paragraph{Normalizing Flows}
We implement the invertible functions as linear functions or piecewise linear functions, which are spline functions. To construct more expressive invertible functions for better performance, piecewise-quadratic functions \cite{durkan2019neural} or cubic splines \cite{durkan2019neural} or even invertible neural networks \cite{huang2018neural} can be an option. In addition to designing invertible functions, \citet{dinh2019rad} generalize Eq.(\ref{equation:1}) to piecewise invertible functions and \citet{grathwohl2018ffjord} propose a continuous version of normalizing flows. This is also a potential research direction.

\paragraph{Normalizing Flows Embedding}
The similarity of two random variables can be measured by $f$-divergence \cite{gibbs2002choosing} or Wasserstein distance. However, these metrics all involve computing a definite integral, which may have no closed form. This may limit the introduction of uncertainty. Since the 1-dimensional Wasserstein distance of two piecewise linear distributions always has a closed form, one solution is to approximate any distribution with a piecewise linear distribution \cite{petersen2018optimal}. For example, we can approximate a normal distribution with a triangular distribution. In order to compute the Wasserstein distance efficiently, we decompose $n$-dimensional Wasserstein distance into a sum of 1-dimensional Wasserstein distance by Corollary \ref{corollary:1}, a sufficient condition of Proposition \ref{proposition:1}. We can design other invertible functions such that $p_{\bm{rh}}(\bm{x}_0)$ and $p_{\bm{t}}(\bm{x}_0)$ share the same copula. We choose the initial random variable $\bm{x}_0$ as a continuous random variable, $\bm{x}_0$ can also be a discrete distribution. How to choose a suitable $\bm{x}_0$ is also worth exploring.


\paragraph{Symmetric Groups}
The view of symmetric groups gives us a unified perspective of embedding. We unify embedding as permutations of a set. On one hand, we can easily parameterize a complex permutation and obtain a complex object by leveraging a permutation to act on a simple object. On the other hand, the group operation of symmetric groups, the composition of functions, is easy to compute. To introduce different properties, we can choose symmetric groups of different sets. For example, the set of points, random variables, hyper-rectangles \cite{abboud2020boxe}, manifolds \cite{xiao2016one} and groups \cite{ebisu2018toruse}. Figure \ref{figure:1} shows an example.

\begin{table*}[t]
\begin{center}
\begin{small}
\begin{tabular}{llllllllllll}
\toprule
\multicolumn{3}{c}{}
&\multicolumn{3}{c}{\bf WN18RR}&\multicolumn{3}{c}{\bf FB15k-237}  &\multicolumn{3}{c}{\bf YAGO3-10}\\
\cmidrule(r){4-6}\cmidrule(r){7-9}\cmidrule(r){10-12}\\
\multicolumn{3}{c}{} &MRR &H@1 &H@10 &MRR &H@1 &H@10 &MRR &H@1 &H@10\\
\midrule
\multicolumn{3}{l}{TransE} &0.218  &0.036  &0.506  &0.335  &0.240  &0.526  &0.539  &0.455  &0.691 \\
\multicolumn{3}{l}{RotatE} &0.476  &0.428  &0.571  &0.338  &0.241  &0.533  & 0.495 &0.402 &0.670\\
\midrule
\multicolumn{3}{l}{DistMult} &0.396  &0.379  &0.427  &0.289  &0.206  &0.452  &0.536  &0.471  &0.652\\
\multicolumn{3}{l}{ComplEx} &0.428  &0.440  &0.410  &0.247  &0.158  &0.510  &0.360  &0.260  &0.550 \\
\multicolumn{3}{l}{QuatE} &\underline{0.481}  &0.436  &0.564  &0.311  &0.221  &0.495 &\multicolumn{1}{c}{---}  &\multicolumn{1}{c}{---} &\multicolumn{1}{c}{---}\\
\multicolumn{3}{l}{TuckER} &0.470  &\underline{0.443}  &0.526  &\textbf{0.358}  &\textbf{0.266}  &\textbf{0.544} &\multicolumn{1}{c}{---}  &\multicolumn{1}{c}{---} &\multicolumn{1}{c}{---}\\
\midrule
\multicolumn{3}{l}{ConvE} &0.430 &0.440 &0.520  &0.325  &0.237 &0.501 &0.440 &0.350 &0.620\\
\multicolumn{3}{l}{HypER} &0.465  &0.436  &0.522 &0.341  &0.252  &0.520  &0.533  &0.455  &0.678 \\
\midrule
\multicolumn{3}{l}{DihEdral} &0.480 &\textbf{0.452} &0.536  &0.325  &0.237  &0.501  &0.440  &0.350  &0.620  \\
\multicolumn{3}{l}{NagE} &0.477 &0.432 &\underline{0.574}  &0.340  &0.244 &0.530 &\multicolumn{1}{c}{---}  &\multicolumn{1}{c}{---} &\multicolumn{1}{c}{---}\\
\midrule
\multicolumn{3}{l}{NFE-2} &0.476 &0.431  &0.569  &0.352  &0.256  &0.542  &\underline{0.563}&0.489 &\textbf{0.699}\\
\multicolumn{3}{l}{NFE-1} &\textbf{0.483}  &0.438  &\textbf{0.576}  &\underline{0.355}  &\underline{0.261} &\underline{0.543}  &\textbf{0.570}  &\textbf{0.498}  &\underline{0.697}\\
\multicolumn{3}{l}{NFE-w/o-uncertainty} &0.475  &0.430  &0.568  &0.352  &0.257  &0.542  &\underline{0.563}  &\underline{0.490}  &0.693 \\
\multicolumn{3}{l}{NFE-1-$\mu$} &0.218  &0.036  &0.506  &0.335  &0.240  &0.526  &0.539  &0.455  &0.691\\
\multicolumn{3}{l}{NFE-1-$\sigma$} &0.286  &0.126  &0.522  &0.342  &0.247  &0.531  &0.519  &0.429  &0.676\\
\bottomrule
\end{tabular}
\end{small}
\end{center}
\caption{Knowledge graph completion results on WN18RR, FB15k-237 and YGAO3-10 datasets.}
\label{table:1}
\end{table*}

\begin{table*}[t]
\begin{center}
\begin{small}
\begin{tabular}{llllllllllll}
\toprule
$1/k^2$ &0 &1/16 &1/8 &1/4 &1/2 &1 &2 &4 &8 &16\\
\midrule
MRR   &0.475 &0.478 &0.478 &0.478 &0.479 &0.483 &0.483 &0.483 &0.482 &0.482\\
H@1   &0.430 &0.433 &0.433 &0.433 &0.434 &0.438 &0.438 &0.436 &0.436 &0.436\\
H@10  &0.568 &0.572 &0.573 &0.571 &0.573 &0.576 &0.575 &0.575 &0.572 &0.572\\
\bottomrule
\end{tabular}
\end{small}
\end{center}
\caption{The performance of NFE-1 with different $1/k^2$ on WN18RR dataset. Higher value of $1/k^2$ indicates the model focusing more on uncertainty.}
\label{table:2}
\end{table*}

\section{Related Work}
\paragraph{Group Embedding}
TorusE \cite{ebisu2018toruse} defines embedding in an $n$-dimensional torus space which is a compact Lie group. MobiusE \cite{chen2021mobiuse} embeds entities/relations to the surface of a Mobius ring. \citet{cai2019group} show that logical rules have natural correspondence to  group representation theory. DihEdral \cite{xu2019relation} models the relations with the representation of dihedral group. NagE \cite{yang2020nage} finds the hidden group algebraic structure of relations and embeds relations as group elements. NFE embeds entities/relations as elements of a permutation group.
\paragraph{Uncertainty}
To model uncertainty in KGs, KG2E \cite{he2015learning} represents entities/relations as random variables with multivariate normal distributions. TransG \cite{xiao2016transg} embeds entities as random variables with normal distributions and draws a mixture of normal distribution for relation embedding to handle multiple semantic issue. NFE introduces uncertainty by embedding entities/relations as permutations of a set of random variables.
\paragraph{Normalizing Flows}
Normalizing Flows should satisfy two conditions in order to be practical, the invertible function has tractable inverse and the determinant of Jacobian is easy to compute. A basic form of normalizing flows can be element-wise invertible functions. NICE \cite{dinh2014nice} and RealNVP \cite{dinh2016density} utilize affine functions to construct coupling layers. \citet{muller2019neural} propose a powerful generalization of affine functions, based on monotonic piecewise polynomials. \citet{ziegler2019latent} introduce an invertible non-linear squared function. \citet{durkan2019neural} model the coupling function as a monotone rational-quadratic spline. \citet{jaini2019sum} propose a strictly increasing polynomial and prove such polynomials can approximate any strictly monotonic univariate continuous function.

\section{Experiments}
In this section, we first introduce the experimental settings and compare NFE with existing models. We then show the effectiveness of introducing uncertainty. Finally, we conduct ablation studies. Please see Appendix D for more experimental details.

\subsection{Experimental Settings}
\label{section:5.1}
\paragraph{Datasets}
We evaluate our model on three popular knowledge graph completion datasets, WN18RR \cite{dettmers2018convolutional}, FB15k-237 \cite{toutanova2015representing} and YAGO3-10 \cite{dettmers2018convolutional}.
\paragraph{Evaluation Metric}
We use the filtered MR, MRR and Hits@N (H@N) \citep{bordes2013translating} as evaluation metrics and choose the hyper-parameters with the best filtered MRR on the validation set.
\paragraph{Baselines}
We compare the performance of NFE with several translational models, including  TransE \citep{bordes2013translating}, RotatE \citep{sun2018rotate}, bilinear models, including DistMult \citep{yang2014embedding}, ComplEx \citep{toutanova2015representing}, QuatE \citep{zhang2019quaternion}, TuckER \citep{balavzevic2019tucker}, neural networks models, including ConvE \citep{dettmers2018convolutional}, HypER \citep{balavzevic2019hypernetwork}, and group embedding models, including DihEdral \cite{xu2019relation}, NagE \cite{yang2020nage}.

\subsection{Results}
\label{section:5.2}
Due to the great generality, our proposed NFE is able to have different instantiations, Eq.(\ref{equation:10}) and Eq.(\ref{equation:11}) and Eq.(\ref{equation:12}). We denote Eq.(\ref{equation:10}) as NFE-1. Eq.(\ref{equation:11}) and Eq.(\ref{equation:12}) achieve almost same result, we denote them as NFE-2. We compare the performance of NFE-1 and NFE-2 with baseline models. See Table \ref{table:1} for the results. NFE-1 and NFE-2 achieve state-of-the-art performance on three datasets, especially on YAGO3-10, which is the largest dataset. NFE-2 is slightly inferior to NFE-1. Although NFE-2 is more expressive than NFE-1, NFE-2 may be more difficult to optimize. NFE-1 are derived from TransE and DistMult, but NFE-1 outperforms TransE and DistMult significantly on all metrics on three datasets. NFE-1 is better than neural networks models, ConvE and HypER, and group embedding models, DihEdral and NagE. The results show the effectiveness of our model.

\subsection{Uncertainty}
\label{section:5.3}
Here we focus on the NFE-w/o-uncertainty in Table \ref{table:1} and Table \ref{table:2}. Proposition \ref{proposition:5} shows that NFE-1 can be seen as a generalization of existing models to model uncertainty. NFE-1 can reduces to Eq.(\ref{equation:13}), which do not model uncertainty. We denote Eq.(\ref{equation:13}) as NFE -w/o-uncertainty. Table \ref{table:1} shows that NFE-1 achieves better performance than NFE-w/o-uncertainty on all metrics on three datasets. Thus, the results show the effectiveness of introducing uncertainty.

Proposition \ref{proposition:5} shows that $1/k^2$ in Eq.(\ref{equation:13}) reflects the degree to which Eq.(\ref{equation:13}) focuses on uncertainty. Higher value of $1/k^2$ indicates the model focusing more on uncertainty. Table \ref{table:2} shows the performance of Eq.(\ref{equation:13}) with different $1/k^2$ on WN18RR dataset. The results show that the performance of Eq.(\ref{equation:13}) generally gets worse as $1/k^2$ gets smaller. This also shows the effectiveness of introducing uncertainty.

\subsection{Ablation Study}
\label{section:5.4}
The invertible functions of NFE-1 is Eq.(\ref{equation:4}), the composition of function of TransE and function of DistMult. We conduct ablation studies to analyze the performance of NFE-1 only using one of the functions. We denote Eq.(\ref{equation:10}) using function of TransE as NFE-1-$\mu$, Eq.(\ref{equation:10}) using function of DistMult as NFE-1-$\sigma$. Table \ref{table:1} shows that the performance of NFE-1 is better than NFE-1-$\mu$ and NFE-1-$\sigma$. The reason is that NFE-1 is more expressive than NFE-1-$\mu$ and NFE-1-$\sigma$.

\section{Conclusion}
We propose a unified perspective of embedding and introduce uncertainty into KGE from the view of group theory. We embed entities/relations as elements of a  symmetric group to introduce uncertainty. Based on the embedding, NFE is defined by measuring the similarity of two normalizing flows. Experiments verify the effectiveness of NFE.
\section*{Acknowledgements}
This work is supported by the National Key Research and Development Program of China (2020AAA0106000), the National Natural Science Foundation of China (U19A2079), and the CCCD Key Lab of Ministry of Culture and Tourism.

\bibliography{aaai23}

\appendix

\onecolumn
\setcounter{proposition}{0}
\setcounter{lemma}{0}

\section*{Appendix}
The Appendix is structured as follows:
\begin{enumerate}
    \item In Appendix Remarks, we show some remarks of our model.
    \item In Appendix Unified Representation, we show the models that can be represented as the unified representation.
    \item In Appendix Proofs, we show the proofs.
    \item In Appendix Experimental Details, we show the experimental details.
\end{enumerate}

\section{Remarks}
\label{appendix:1}
\begin{remark}[Correspondence of group notions and the unified representation]
Group theory is a language to describe the symmetries. We define a map from an group element $\bm{g}$ to an invertible function $f_{\bm{g}}$. For every $f_{\bm{g}}$, $\bm{x}$ and $f_{\bm{g}}(\bm{x})$ are symmetrical about $f_{\bm{g}}$, $i.e.$, $\bm{x}$ and $f_{\bm{g}}(\bm{x})$ belong to the same orbit.

A group is a set equipped with a binary operation, in such a way that the operation is closed and associative, an identity element exists and every element has an inverse. The properties of groups have some correspondence to Eq.(\ref{equation:2}). The closure property is to ensure that for any $f_{\bm{h}}(\bm{x})$ and $f_{\bm{r}}(\bm{x})$ there exists a $f_{\bm{t}}(\bm{x})$ such that $f_{\bm{t}}(\bm{x})=f_{\bm{r}}\circ f_{\bm{h}}(\bm{x})$ for any $\bm{x}$. The associativity corresponds to the functional composition of $f_{\bm{r}}$ and $f_{\bm{h}}$. The existence of identity element corresponds to the existence of an identity map. The existence of inverse element is ensure that the functions $f_{\bm{h}},f_{\bm{r}}$ and $f_{\bm{t}}$ invertible. In the original definition of DistMult, it is still valid if the entries of $\bm{g}$ has zero entries. We can additionally define the group action $\alpha$ for vectors with zero entries. In other words, we do not need to constraint the function $f_{\bm{g}}$ invertible.
\end{remark}
\begin{remark}[Measurement of  uncertainty]
The uncertainty of a continuous random variable $\bm{Y}\sim p(\bm{y})$ can be measured by the differential entropy: $$H(\bm{Y})=-\int p(\bm{y})\log p(\bm{y})\mathrm{d}\bm{y}$$
Eq.(\ref{equation:3}) involves four random variables, $\bm{x}_0$, $f_{\bm{h}}(\bm{x}_0)$, $f_{\bm{r}}\circ f_{\bm{h}}(\bm{x}_0)$ and $f_{\bm{t}}(\bm{x}_0)$. Thus, $H(f_{\bm{h}}(\bm{x}_0))-H(\bm{x}_0)$,$H(f_{\bm{r}}\circ f_{\bm{h}}(\bm{x}_0))-H(f_{\bm{h}}(\bm{x}_0))$ and $H(f_{\bm{t}}(\bm{x}_0))-H(\bm{x}_0)$ measure the amount of change of uncertainty.
\end{remark}
\begin{remark}[Eq.(\ref{equation:5}) or Eq.(\ref{equation:7}) for non-invertible functions]
Eq.(\ref{equation:1}) is only valid for invertible functions. Thus, we need to constraint the entries of $\bm{g}$ of Eq.(\ref{equation:4}) non-zero. To ensure Eq.(\ref{equation:5}) invertible, we need to constraint the entries of $\bm{h}_{\sigma},\bm{r}_{\sigma}$ and $\bm{t}_{\sigma}$ non-zero. If $\bm{x}_0\sim \mathcal{N}(\bm{0},\bm{I})$, the invertible functions are Eq.(\ref{equation:5}) and similarity metric is KL divergence, then the scoring function is $\infty$ if there exists entries of $\bm{h}_{\sigma}$ or $\bm{r}_{\sigma}$ or $\bm{t}_{\sigma}$ is zero. While the Wasserstein distance is still valid when the entries of $\bm{g}$ are zero. Thus, another benefit of Wasserstein distance over KL divergence is that Wasserstein distance is more numerically stable.
\end{remark}
\begin{remark}[KL divergence similarity metric]
If $\bm{x}_0\sim U[-\sqrt{3},\sqrt{3}]^n$, the invertible functions is Eq.(\ref{equation:5}) or Eq.(\ref{equation:7}) and similarity metric is Eq.(\ref{equation:8}), then the scoring function is not valid. If $\bm{x}_0\sim \mathcal{N}(\bm{0},\bm{I})$, the invertible functions is Eq.(\ref{equation:5}), the similarity metric is Eq.(\ref{equation:8}), then the scoring function is
\begin{align}
    f(\bm{h}, \bm{r}, \bm{t})=\sum_{i=1}^{n}\log\frac{|\bm{t}_{\sigma}|_i}{|\bm{r}_{\sigma}|_i\odot|\bm{h}_{\sigma}|_i}+\frac{|\bm{r}_{\sigma}|_i^2\odot|\bm{h}_{\sigma}|_i^2+(\bm{r}_{\sigma i}\odot\bm{h}_{\mu i}+\bm{r}_{\mu i}-\bm{t}_{\mu i})^2}{2|\bm{t}_{\sigma}|_i^2}-\frac{n}{2}
    \label{equation:15}
\end{align}
If $\bm{x}_0\sim \mathcal{N}(\bm{0},\bm{I})$, the invertible functions is Eq.(\ref{equation:7}), the similarity metric is Eq.(\ref{equation:8}), then the scoring function has no closed form.
\end{remark}

\section{Unified Representation}
\label{appendix:2}
\begin{figure}[t]
\begin{center}
\centerline{\includegraphics[width=0.5\textwidth]{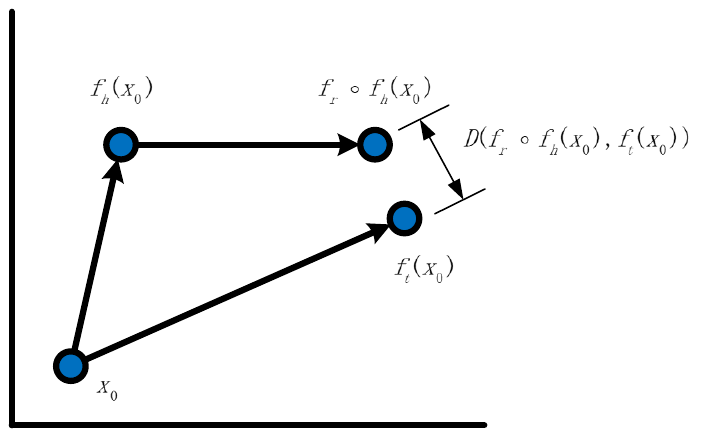}}
\caption{An illustration of Eq.(\ref{equation:2}).}
\label{figure:2}
\end{center}
\end{figure}
We have proposed a unify representation of TransE and DistMult, Eq.(\ref{equation:2}). Eq.(\ref{equation:2}) can be illustrated by Figure \ref{figure:2}.  Eq.(\ref{equation:2}) is composed of three parts, the initial object $\bm{x}_{0}\in X$, the invertible functions $(f_{\bm{h}}, f_{\bm{r}},f_{\bm{t}})$ on a set $X$ and the similarity metric $D(\cdot,\cdot)$. In addtion to TransE and DistMult, we show other models that can be represented as Eq.(\ref{equation:2}), RotatE \citep{sun2018rotate}, TorusE \cite{ebisu2018toruse}, RESCAL \cite{nickel2011three}, ComplEx \citep{toutanova2015representing}. We summarize in Table \ref{table:3}.

The scoring function of RotatE is defined as:
$$f(\bm{h},\bm{r},\bm{t})=-\|\bm{h}\odot\bm{r}-\bm{t}\|_2^2,\bm{h},\bm{r},\bm{t}\in \mathbb{C}^n,|\bm{r}|=1$$
The scoring function of TorusE is defined as:
$$f(\bm{h},\bm{r},\bm{t})=-\|[\bm{h}]+[\bm{r}]-[\bm{t}]\|_2^2,[\bm{h}],[\bm{r}],[\bm{t}]\in \mathbb{T}^n$$
The scoring function of RESCAL is defined as:
$$f(\bm{h},\bm{R},\bm{t})=\bm{h}^{T}\bm{R}\bm{t},\bm{h},\bm{t}\in \mathbb{R}^n,\bm{R}\in \mathbb{R}^{n\times n}$$
The scoring function of ComplEx is defined as:
$$f(\bm{h},\bm{r},\bm{t})=\text{Re}(\langle \bm{h}, \bm{r}, \bm{t} \rangle),\bm{h},\bm{r},\bm{t}\in \mathbb{C}^n$$

\citet{balavzevic2019tucker} show that DistMult, ComplEx and QuatE are special cases of TuckER. The scoring function of TuckER is defined as:

\begin{align*}
    f(\bm{h},\bm{r},\bm{t})= \sum_{i=1}^{n}\sum_{j=1}^{n}\sum_{k=1}^{n}\bm{W}_{ijk}\bm{h}_{i}\bm{r}_{j}\bm{t}_{k}
\end{align*}
where $\bm{W}\in \mathbb{R}^{n\times n\times n}$ is the weight tensor. We can similarly generalize Eq.(\ref{equation:2}) to the form
\begin{align*}
  f(\bm{h}, \bm{r}, \bm{t})= \sum_{i=1}^{n}\sum_{j=1}^{n}\sum_{k=1}^{n}\bm{W}_{ijk}D(f_{\bm{r}_{j}}\circ f_{\bm{h}_{i}}(\bm{x_{0}}), f_{\bm{t}_{k}}(\bm{x_{0}}))
\end{align*}

\begin{table}[t]
\caption{Examples of the unify representation, Eq.(\ref{equation:2})}
\label{table:3}
\begin{center}
\begin{small}
\begin{tabular}{llllll}
\toprule
Model &$\bm{x}_0$ &$f_{\bm{h}}(\bm{x})$ &$f_{\bm{r}}(\bm{x})$ &$f_{\bm{t}}(\bm{x})$ &$D(\cdot,\cdot)$\\
\midrule
TransE    &$\bm{0}\in\mathbb{R}^n$ &$\bm{h}+\bm{x}$ &$\bm{r}+\bm{x}$ &$\bm{t}+\bm{x}$ &$-\|\cdot-\cdot\|$\\
RotatE    &$\bm{1}\in \mathbb{C}^n$ &$\bm{h}\odot\bm{x}$ &$\bm{r}\odot\bm{x}$ &$\bm{t}\odot\bm{x}$ &$-\|\cdot-\cdot\|$\\
TorusE &$[\bm{0}]\in\mathbb{T}^n$ &$[\bm{h}]+[\bm{x}]$ &$[\bm{r}]+[\bm{x}]$ &$[\bm{t}]+[\bm{x}]$ &$-\|\cdot-\cdot\|$\\
DistMult &$\bm{1}\in\mathbb{R}^n$ &$\bm{h}\odot\bm{x}$ &$\bm{r}\odot\bm{x}$ &$\bm{t}\odot\bm{x}$ &$\langle \cdot,\cdot \rangle$\\
RESCAL &$\bm{1}\in\mathbb{R}^n$ &$\bm{h}\odot\bm{x}$ &$\bm{r}^{T}\bm{x}$ &$\bm{t}\odot\bm{x}$ &$\langle \cdot,\cdot \rangle$\\
ComplEx  &$\bm{1}\in \mathbb{C}^n$ &$\bm{h}\odot\bm{x}$ &$\bm{r}\odot\bm{x}$ &$\bm{t}^{*}\odot\bm{x}$ &\text{Re}($\langle \cdot,\cdot \rangle$)\\
NFE-1 &$\mathcal{N}(\bm{0},\bm{I})$ &$\bm{h}_{\sigma}\odot\bm{x}+\bm{h}_{\mu}$ &$\bm{r}_{\sigma}\odot\bm{x}+\bm{r}_{\mu}$ &$\bm{t}_{\sigma}\odot\bm{x}+\bm{t}_{\mu}$ &Eq.(\ref{equation:10})\\
NFE-2 &$\mathcal{N}(\bm{0},\bm{I})$ &$
   \bm{h}_{\sigma_{1}}\odot\bm{x}+\bm{h}_{\mu}, \bm{x}\leq \bm{0}$ & $\bm{r}_{\sigma}\odot\bm{x}+\bm{r}_{\mu}$
   &$\bm{t}_{\sigma_{1}}\odot\bm{x}+\bm{t}_{\mu}, \bm{x}\leq \bm{0}$ &Eq.(\ref{equation:11})\\
   & &$\bm{h}_{\sigma_{2}}\odot\bm{x}+\bm{h}_{\mu}, \bm{x}> \bm{0}$
   & &$\bm{t}_{\sigma_{2}}\odot\bm{x}+\bm{t}_{\mu}, \bm{x}> \bm{0}$ &\\
\bottomrule
\end{tabular}
\end{small}
\end{center}
\end{table}

\section{Proofs}
\label{appendix:3}
\begin{proposition}
For two PDFs, $p(\bm{x})$ and $q(\bm{y})$, $W(p,q)=\sum_{i=1}^{n}W(p_i,q_i)$ iff $p(\bm{x})$ and $q(\bm{y})$ share the same copula, where $p_i$ and $q_i$ are the marginal distributions of $\bm{x}_{i}$ and $\bm{y}_{i}$, respectively \citep{cuestaalbertos1993optimal}.
\end{proposition}
See Theorem 2.9 of \citep{cuestaalbertos1993optimal} for the proof. A copula is a multivariate cumulative distribution function for which the marginal probability distribution of each variable is uniform on the interval $[0.1]$. Copulas are used to describe/model the dependence (inter-correlation) between random variables. The copula of $(X_1,X_2,\cdots,X_n)$ is defined as the joint cumulative distribution function of $(U_1,U_2,\cdots,U_n)$:
$C(U_1,U_2,\cdots,U_n)=\text{Pr}(U_1\leq u_1,U_2\leq u_2,\cdots,U_n\leq u_n)$. If $X_1,X_2,\cdots,X_n$ are independent, then $C(U_1,U_2,\cdots,U_n)=\prod_{i=1}^{n}U_i$. Thus, we have the following corollary.
\begin{corollary}
For two PDFs, $p(\bm{x})$ and $q(\bm{y})$, $W(p,q)=\sum_{i=1}^{n}W(p(\bm{x}_i),q(\bm{y}_i))$ if $p(\bm{x})=\prod_{i=1}^{n}p(\bm{x}_i)$ and $q(\bm{y})=\prod_{i=1}^{n}q(\bm{y}_i)$.
\end{corollary}
\begin{proposition}
If $\bm{x}_0\sim U[-\sqrt{3},\sqrt{3}]^n$ or $\bm{x}_0\sim \mathcal{N}(\bm{0},\bm{I})$, the invertible functions are Eq.(\ref{equation:5}) and similarity metric is Eq.(\ref{equation:9}), then the scoring function is
\begin{align*}
f(\bm{h}, \bm{r}, \bm{t})=-\|\bm{r}_{\sigma}\odot\bm{h}_{\mu}+\bm{r}_{\mu}-\bm{t}_{\mu}\|_2^2-\||\bm{r}_{\sigma}\odot\bm{h}_{\sigma}|-|\bm{t}_{\sigma}|\|_2^2
\end{align*}
\end{proposition}
\begin{proof}
We first prove the case of $\bm{x}_0\sim U[-\sqrt{3},\sqrt{3}]^n$.
We show that for two 1-dimensional normal distribution $p(x)=U[a_{1},b_{1}]=U[\mu_{1}-\sqrt{3}\sigma_{1},\mu_{1}+\sqrt{3}\sigma_{1}]$ and $q(x)=U[a_{2},b_{2}]=U[\mu_{2}-\sqrt{3}\sigma_{2},\mu_{2}+\sqrt{3}\sigma_{2}]$, the Wasserstein distance is equal to $W(p,q)=\|\mu_{1}-\mu_{2}\|^2+\|\sigma_{1}-\sigma_{2}\|^2$.
1-dimensional Wasserstein distance is equal to
$$W(p,q)=\int_{0}^{1}|F^{-1}(z)-G^{-1}(z)|\mathrm{d}z$$
where $F^{-1}(z)$ and $G^{-1}(z)$ are the inverse cumulative distribution function of $p(\bm{x})$ and $q(\bm{x})$, $i.e.$, $F^{-1}(z)=a_{1}+(b_{1}-a_{1})z$ and $G^{-1}(z)=a_{2}+(b_{2}-a_{2})z$. Thus we have that
\begin{align*}
    W(p,q)=&\int_{0}^{1}|F^{-1}(z)-G^{-1}(z)|^2\mathrm{d}z\\
    =&\int_{0}^{1}((a_{1}-a_{2})+z(b_{1}-b_{2}-a_{1}+a_{2}))^2\mathrm{d}z\\
    =&\int_{0}^{1}(a_{1}-a_{2})^2\mathrm{d}z+\int_{0}^{1}z^2(b_{1}-b_{2}-a_{1}+a_{2})^2\mathrm{d}z\\
    +&\int_{0}^{1}2(a_{1}-a_{2})(b_{1}-b_{2}-a_{1}+a_{2})z\mathrm{d}z\\
    =&(a_{1}-a_{2})^2+\frac{1}{3}(b_{1}-b_{2}-a_{1}+a_{2})^2
    +(a_{1}-a_{2})(b_{1}-b_{2}-a_{1}+a_{2})\\
    =&\frac{1}{3}((a_{1}-a_{2})^2+(b_{1}-b_{2})^2+(a_{1}-a_{2})(b_{1}-b_{2}))\\
    =&(\mu_{1}-\mu_{2})^2+(\sigma_{1}-\sigma_{2})^2
\end{align*}
For two uniform distribution $U[\bm{\mu}_1-\sqrt{3}\bm{\sigma}_1,\bm{\mu}_1+\sqrt{3}\bm{\sigma}_1]$ and $U[\bm{\mu}_2-\sqrt{3}\bm{\sigma}_2,\bm{\mu}_2+\sqrt{3}\bm{\sigma}_2]$ which satisfy $p(\bm{x})=\prod_{i=1}^{n}p(\bm{x}_i)$ and $q(\bm{x})=\prod_{i=1}^{n}q(\bm{x}_i)$, by Corollary \ref{corollary:1}, the Wasserstein distance is equal to $$W(p(\bm{x}),q(\bm{x}))=\sum_{i=1}^{n}W(p(\bm{x}_i),q(\bm{x}_i))=\|\bm{\mu}_{1}-\bm{\mu}_{2}\|^2+\|\bm{\sigma_{1}}-\bm{\sigma_{2}}\|^2$$
This result is same the result in \cite{dowson1982frechet}. For every triplet $(\bm{h}, \bm{r}, \bm{t})$, we have that
\begin{align*}
f_{\bm{h}}(\bm{x})=\bm{h}_{\sigma}\odot\bm{x}+\bm{h}_{\mu},f_{\bm{r}}(\bm{x})=\bm{r}_{\sigma}\odot\bm{x}+\bm{r}_{\mu},f_{\bm{t}}(\bm{x})=\bm{t}_{\sigma}\odot\bm{x}+\bm{t}_{\mu}\\
f_{\bm{r}}\circ f_{\bm{h}}(\bm{x})=\bm{r}_{\sigma}\odot\bm{h}_{\sigma}\odot\bm{x}+\bm{r}_{\sigma}\odot\bm{h}_{\mu}+\bm{r}_{\mu},f_{\bm{t}}(\bm{x})=\bm{t}_{\sigma}\odot\bm{x}+\bm{t}_{\mu}
\end{align*}
Since $\bm{x}_0\sim U[-\sqrt{3},\sqrt{3}]^n$, by Eq.(\ref{equation:1}), we have that $f_{\bm{r}}\circ f_{\bm{h}}(\bm{x}_{0})\sim U[\bm{r}_{\sigma}\odot\bm{h}_{\mu}+\bm{r}_{\mu}-\sqrt{3}|\bm{r}_{\sigma}\odot\bm{h}_{\sigma}|,\bm{r}_{\sigma}\odot\bm{h}_{\mu}+\bm{r}_{\mu}+\sqrt{3}|\bm{r}_{\sigma}\odot\bm{h}_{\sigma}|],f_{\bm{t}}(\bm{x}_{0})\sim U[\bm{t}_{\mu}-\sqrt{3}|\bm{t}_{\sigma_{1}}|,\bm{t}_{\mu}+\sqrt{3}|\bm{t}_{\sigma_{1}}|]$. Then the scoring function is equal to
$$f(\bm{h}, \bm{r}, \bm{t})=-W(f_{\bm{r}}\circ f_{\bm{h}}(\bm{x_{0}}), f_{\bm{t}}(\bm{x_{0}}))=-\|\bm{r}_{\sigma}\odot\bm{h}_{\mu}+\bm{r}_{\mu}-\bm{t}_{\mu}\|_2^2-\||\bm{r}_{\sigma}\odot\bm{h}_{\sigma}|-|\bm{t}_{\sigma}|\|_2^2$$
\end{proof}

We then prove the case of $\bm{x}_0\sim \mathcal{N}(\bm{0},\bm{I})$. We first compute the indefinite integral of $\int\erf^{-1}(x)\mathrm{d}x$ and $\int\erf^{-1}(x)^2\mathrm{d}x$:
\begin{align*}
\int\erf^{-1}(x)\mathrm{d}x=&\int\frac{2}{\sqrt{\pi}}e^{-\erf^{-1}(x)^2}\erf^{-1}(x)\frac{\sqrt{\pi}}{2}e^{\erf^{-1}(x)^2}\mathrm{d}x\\
=&\int\frac{2}{\sqrt{\pi}}e^{-\erf^{-1}(x)^2}\erf^{-1}(x)\mathrm{d}\erf^{-1}(x)\\
=&\int-\frac{1}{\sqrt{\pi}}\mathrm{d}e^{-\erf^{-1}(x)^2}\\
=&-\frac{1}{\sqrt{\pi}}e^{-\erf^{-1}(x)^2}
\end{align*}
\begin{align*}
\int\erf^{-1}(x)^2\mathrm{d}x=&\int\frac{1}{\sqrt{\pi}}\erf^{-1}(x)e^{-\erf^{-1}(x)^2}2\erf^{-1}(x)\frac{\sqrt{\pi}}{2}e^{\erf^{-1}(x)^2}\mathrm{d}x\\
=&\int-\frac{1}{\sqrt{\pi}}\erf^{-1}(x)\mathrm{d}e^{-\erf^{-1}(x)^2}\\
=&-\frac{1}{\sqrt{\pi}}\erf^{-1}(x)e^{-\erf^{-1}(x)^2}+\int\frac{1}{\sqrt{\pi}}e^{-\erf^{-1}(x)^2}\mathrm{d}\erf^{-1}(x)\\
=&-\frac{1}{\sqrt{\pi}}\erf^{-1}(x)e^{-\erf^{-1}(x)^2}+\int\frac{1}{\sqrt{\pi}}e^{-\erf^{-1}(x)^2}\frac{\sqrt{\pi}}{2}e^{\erf^{-1}(x)^2}\mathrm{d}x\\
=&\frac{x}{2}-\frac{1}{\sqrt{\pi}}\erf^{-1}(x)e^{-\erf^{-1}(x)^2}
\end{align*}
Then we show that for two 1-dimensional normal distribution $p(x)=\mathcal{N}(\mu_{1},\sigma_{1})$ and $q(x)=\mathcal{N}(\mu_{2},\sigma_{2})$, the Wasserstein distance is equal to $W(p,q)=(\mu_{1}-\mu_{2})^2+(\sigma_{1}-\sigma_{2})^2$.
The 1-dimensional Wasserstein distance is equal to
$$W(p,q)=\int_{0}^{1}|F^{-1}(z)-G^{-1}(z)|\mathrm{d}z$$
where $F^{-1}(z)$ and $G^{-1}(z)$ are the inverse cumulative distribution function of $p(x)$ and $q(x)$, $i.e.$, $F^{-1}(z)=\mu_{1}+\sqrt{2}\sigma_{1}\erf^{-1}(2z-1)$ and $G^{-1}(z)=\mu_{2}+\sqrt{2}\sigma_{2}\erf^{-1}(2z-1)$. Therefore
\begin{align*}
    W(p,q)=&\int_{0}^{1}|F^{-1}(z)-G^{-1}(z)|^2\mathrm{d}z\\
    =&\int_{0}^{1}((\mu_{1}-\mu_{2})+\sqrt{2}\erf^{-1}(2z-1)(\sigma_{1}-\sigma_{2}))^2\mathrm{d}z\\
    =&\int_{0}^{1}(\mu_{1}-\mu_{2})^2\mathrm{d}z+\int_{0}^{1}2\erf^{-1}(2z-1)^2(\sigma_{1}-\sigma_{2})^2\mathrm{d}z\\
    +&\int_{0}^{1}2\sqrt{2}(\mu_{1}-\mu_{2})\erf^{-1}(2z-1)(\sigma_{1}-\sigma_{2})\mathrm{d}z\\
    =&(\mu_{1}-\mu_{2})^2+(\sigma_{1}-\sigma_{2})^2\int_{-1}^{1}\erf^{-1}(z)^2\mathrm{d}z
    +\sqrt{2}(\mu_{1}-\mu_{2})(\sigma_{1}-\sigma_{2})\int_{-1}^{1}\erf^{-1}(z)\mathrm{d}z\\
    =&(\mu_{1}-\mu_{2})^2+(\sigma_{1}-\sigma_{2})^2
\end{align*}
For two normal distribution $p(\bm{x})=\mathcal{N}(\bm{\mu}_{1},\bm{\sigma_{1}})$ and $q(\bm{x})=\mathcal{N}(\bm{\mu}_{2},\bm{\sigma_{2}})$ which satisfy $p(\bm{x})=\prod_{i=1}^{n}p(\bm{x}_i)$ and $q(\bm{x})=\prod_{i=1}^{n}q(\bm{x}_i)$, by Corollary \ref{corollary:1}, the Wasserstein distance is equal to $$W(p(\bm{x}),q(\bm{x}))=\sum_{i=1}^{n}W(p(\bm{x}_i),q(\bm{x}_i))=\|\bm{\mu}_{1}-\bm{\mu}_{2}\|^2+\|\bm{\sigma_{1}}-\bm{\sigma_{2}}\|^2$$
This result is same the result in \cite{dowson1982frechet}. For every triplet $(\bm{h}, \bm{r}, \bm{t})$, we have that
\begin{align*}
f_{\bm{h}}(\bm{x})=\bm{h}_{\sigma}\odot\bm{x}+\bm{h}_{\mu},f_{\bm{r}}(\bm{x})=\bm{r}_{\sigma}\odot\bm{x}+\bm{r}_{\mu},f_{\bm{t}}(\bm{x})=\bm{t}_{\sigma}\odot\bm{x}+\bm{t}_{\mu}\\
f_{\bm{r}}\circ f_{\bm{h}}(\bm{x})=\bm{r}_{\sigma}\odot\bm{h}_{\sigma}\odot\bm{x}+\bm{r}_{\sigma}\odot\bm{h}_{\mu}+\bm{r}_{\mu},f_{\bm{t}}(\bm{x})=\bm{t}_{\sigma}\odot\bm{x}+\bm{t}_{\mu}
\end{align*}
Since $\bm{x}_{0}\sim \mathcal{N}(\bm{0},\bm{I})$, by Eq.(\ref{equation:1}), we have that $f_{\bm{r}}\circ f_{\bm{h}}(\bm{x}_{0})\sim\mathcal{N}(\bm{r}_{\sigma}\odot\bm{h}_{\mu}+\bm{r}_{\mu},|\bm{r}_{\sigma}\odot\bm{h}_{\sigma}|),f_{\bm{t}}(\bm{x}_{0})\sim\mathcal{N}(\bm{t}_{\mu},|\bm{t}_{\sigma}|)$. Then the scoring function is equal to
$$f(\bm{h}, \bm{r}, \bm{t})=-W(f_{\bm{r}}\circ f_{\bm{h}}(\bm{x_{0}}), f_{\bm{t}}(\bm{x_{0}}))=-\|\bm{r}_{\sigma}\odot\bm{h}_{\mu}+\bm{r}_{\mu}-\bm{t}_{\mu}\|_2^2-\||\bm{r}_{\sigma}\odot\bm{h}_{\sigma}|-|\bm{t}_{\sigma}|\|_2^2$$

\begin{proposition}
If $\bm{x}_0\sim U[-\sqrt{3},\sqrt{3}]^n$, the invertible functions are  Eq.(\ref{equation:7}) and similarity metric is Eq.(\ref{equation:9}), then the scoring function is 
\begin{align*}
f(\bm{h}, \bm{r},\bm{t})=&-\|\bm{r}_{\sigma}\odot\bm{h}_{\mu}+\bm{r}_{\mu}-\bm{t}_{\mu}\|_2^2-\frac{1}{2}\||\bm{r}_{\sigma}\odot\bm{h}_{\sigma_{1}}|-|\bm{t}_{\sigma_{1}}|\|_2^2-\frac{1}{2}\||\bm{r}_{\sigma}\odot\bm{h}_{\sigma_{2}}|-|\bm{t}_{\sigma_{2}}|\|_2^2\notag\\
&-\sqrt{\frac{3}{4}}(\bm{r}_{\sigma}\odot\bm{h}_{\mu}+\bm{r}_{\mu}-\bm{t}_{\mu})^T(|\bm{r}_{\sigma}\odot\bm{h}_{\sigma_{2}}|-|\bm{r}_{\sigma}\odot\bm{h}_{\sigma_{1}}|+|\bm{t}_{\sigma_{1}}|-|\bm{t}_{\sigma_{2}}|)
\end{align*}
\end{proposition}
\begin{proof}
We first show that for two 1-dimensional normal distribution
\begin{align*}
   p(x)=
   \begin{cases}
   U[a_{1},b_{1}]=U[\mu_{1}-\sqrt{3}\sigma_{1},\mu_{1}] &\text{if } a_{1}\leq x\leq b_{1}\\
   U[b_{1},c_{1}]=U[\mu_{1},\mu_{1}+\sqrt{3}\sigma_{2}] &\text{if } b_{1}< x\leq c_{1}
   \end{cases},\\
   q(x)=
   \begin{cases}
   U[a_{2},b_{2}]=U[\mu_{2}-\sqrt{3}\sigma_{3},\mu_{2}] &\text{if } a_{2}\leq x\leq b_{2}\\
   U[b_{2},c_{2}]=U[\mu_{2},\mu_{2}+\sqrt{3}\sigma_{4}] &\text{if } b_{2}< x\leq c_{2}
   \end{cases}
\end{align*}
the Wasserstein distance is equal to $$W(p,q)=(\mu_{1}-\mu_{2})^2+\frac{1}{2}(\sigma_{1}-\sigma_{3})^2+\frac{1}{2}(\sigma_{1}-\sigma_{3})^2+\sqrt{\frac{3}{4}}(\mu_{1}-\mu_{2})(\sigma_{2}-\sigma_{1}+\sigma_{3}-\sigma_{4})$$ 1-dimensional Wasserstein distance is equal to
$$W(p,q)=\int_{0}^{1}|F^{-1}(z)-G^{-1}(z)|\mathrm{d}z$$
where $F^{-1}(z)$ and $G^{-1}(z)$ are the inverse cumulative distribution function of $p(\bm{x})$ and $q(\bm{x})$, $i.e.$,
\begin{align*}
   F^{-1}(z)=
   \begin{cases}
   a_{1}+(b_{1}-a_{1})2z & \text{if }0\leq z\leq \frac{1}{2}\\
   2b_{1}-c_{1}+(c_{1}-b_{1})2z & \text{if }\frac{1}{2}< z\leq 1
   \end{cases}\\
   G^{-1}(z)=
   \begin{cases}
   a_{2}+(b_{2}-a_{2})2z & \text{if }0\leq z\leq \frac{1}{2}\\
   2b_{2}-c_{2}+(c_{2}-b_{2})2z & \text{if }\frac{1}{2}< z\leq 1
   \end{cases}
\end{align*}
Thus we have that
\begin{align*}
    W(p,q)=&\int_{0}^{1}|F^{-1}(z)-G^{-1}(z)|^2\mathrm{d}z\\
    =&\int_{0}^{\frac{1}{2}}((a_{1}-a_{2})+2z(b_{1}-b_{2}-a_{1}+a_{2}))^2\mathrm{d}z\\+&\int_{\frac{1}{2}}^{1}((2b_{1}-2b_{2}-c_{1}+c_{2})+2z(c_{1}-c_{2}-b_{1}+b_{2}))^2\mathrm{d}z\\
    =&\int_{0}^{\frac{1}{2}}(a_{1}-a_{2})^2\mathrm{d}z+\int_{\frac{1}{2}}^{1}(2b_{1}-2b_{2}-c_{1}+c_{2})^2\mathrm{d}z\\
    +&\int_{0}^{\frac{1}{2}}4z^2(b_{1}-b_{2}-a_{1}+a_{2})^2\mathrm{d}z++\int_{\frac{1}{2}}^{1}4z^2(c_{1}-c_{2}-b_{1}+b_{2})^2\mathrm{d}z\\
    +&\int_{0}^{\frac{1}{2}}4z(a_{1}-a_{2})(b_{1}-b_{2}-a_{1}+a_{2})\mathrm{d}z\\
    +&\int_{\frac{1}{2}}^{1}4z(2b_{1}-2b_{2}-c_{1}+c_{2})(c_{1}-c_{2}-b_{1}+b_{2})\mathrm{d}z\\
    =&\frac{1}{2}(a_{1}-a_{2})^2+\frac{1}{2}(2b_{1}-2b_{2}-c_{1}+c_{2})^2\\
    +&\frac{1}{6}(b_{1}-b_{2}-a_{1}+a_{2})^2+\frac{7}{6}(c_{1}-c_{2}-b_{1}+b_{2})^2\\
    +&\frac{1}{2}(a_{1}-a_{2})(b_{1}-b_{2}-a_{1}+a_{2})+\frac{3}{2}(2b_{1}-2b_{2}-c_{1}+c_{2})(c_{1}-c_{2}-b_{1}+b_{2})\\
    =&\frac{1}{6}((a_{1}-a_{2})^2+2(b_{1}-b_{2})^2+(c_{1}-c_{2})^2+(b_{1}-b_{2})(a_{1}-a_{2}+c_{1}-c_{2}))\\
    =&(\mu_{1}-\mu_{2})^2+\frac{1}{2}(\sigma_{1}-\sigma_{3})^2+\frac{1}{2}(\sigma_{1}-\sigma_{3})^2+\sqrt{\frac{3}{4}}(\mu_{1}-\mu_{2})(\sigma_{2}-\sigma_{1}+\sigma_{3}-\sigma_{4})
\end{align*}
By Corollary \ref{corollary:1}, we have that for two distribution
\begin{align*}
   p(\bm{x})=
   \begin{cases}
   U[\bm{a}_{1},\bm{b}_{1}]=U[\bm{\mu}_{1}-\sqrt{3}\bm{\sigma}_{1},\bm{\mu}_{1}] & \text{if }\bm{a}_{1}\leq \bm{x}\leq \bm{b}_{1}\\
   U[\bm{b}_{1},\bm{c}_{1}]=U[\bm{\mu}_{1},\bm{\mu}_{1}+\sqrt{3}\bm{\sigma}_{2}] & \text{if }\bm{b}_{1}< \bm{x}\leq \bm{c}_{1}
   \end{cases},\\
   q(\bm{x})=
   \begin{cases}
   U[\bm{a}_{2},\bm{b}_{2}]=U[\bm{\mu}_{2}-\sqrt{3}\bm{\sigma}_{3},\bm{\mu}_{2}] & \text{if }\bm{a}_{2}\leq \bm{x}\leq \bm{b}_{2}\\
   U[\bm{b}_{2},\bm{c}_{2}]=U[\bm{\mu}_{2},\bm{\mu}_{2}+\sqrt{3}\bm{\sigma}_{4}] & \text{if }\bm{b}_{2}< \bm{x}\leq \bm{c}_{2}
   \end{cases}
\end{align*}
the Wasserstein distance is equal to
\begin{align*}
W(p(\bm{x}),q(\bm{x}))=\sum_{i=1}^{n}W(p(\bm{x}_i),q(\bm{x}_i))=&\|\bm{\mu}_{1}-\bm{\mu}_{2}\|^2+\frac{1}{2}\|\bm{\sigma_{1}}-\bm{\sigma_{3}}\|^2+\frac{1}{2}\|\bm{\sigma_{2}}-\bm{\sigma_{4}}\|^2\\+&\sqrt{\frac{3}{4}}(\bm{\mu}_{1}-\bm{\mu}_{2})^T(\bm{\sigma}_{2}-\bm{\sigma}_{1}+\bm{\sigma}_{3}-\bm{\sigma}_{4})
\end{align*}
For every triplet $(\bm{h}, \bm{r}, \bm{t})$, we have that
\begin{align*}
   f_{\bm{r}}\circ f_{\bm{h}}(\bm{x})=
   \begin{cases}
   \bm{r}_{\sigma}\odot\bm{h}_{\sigma_{1}}\odot\bm{x}+\bm{r}_{\sigma}\odot\bm{h}_{\mu}+\bm{r}_{\mu} & \text{if }\bm{x}\leq \bm{0}\\
   \bm{r}_{\sigma}\odot\bm{h}_{\sigma_{2}}\odot\bm{x}+\bm{r}_{\sigma}\odot\bm{h}_{\mu}+\bm{r}_{\mu} & \text{if }\bm{x}>\bm{0}
   \end{cases}, 
   f_{\bm{t}}(\bm{x})=
   \begin{cases}
   \bm{t}_{\sigma_{1}}\odot\bm{x}+\bm{t}_{\mu} & \text{if }\bm{x}\leq \bm{0}\\
   \bm{t}_{\sigma_{2}}\odot\bm{x}+\bm{t}_{\mu} & \text{if }\bm{x}>\bm{0}
   \end{cases}
\end{align*}
Since $\bm{x}_{0}\sim U[-\sqrt{3},\sqrt{3}]^n$, by Eq.(\ref{equation:1}), we have that
\begin{align*}
   &f_{\bm{r}}\circ f_{\bm{h}}(\bm{x}_{0})\sim
   \begin{cases}
   U[\bm{r}_{\sigma}\odot\bm{h}_{\mu}+\bm{r}_{\mu}-\sqrt{3}|\bm{r}_{\sigma}\odot\bm{h}_{\sigma_{1}}|,\bm{r}_{\sigma}\odot\bm{h}_{\mu}+\bm{r}_{\mu}] & \text{if }x\leq \bm{r}_{\sigma}\odot\bm{h}_{\mu}+\bm{r}_{\mu}\\
   U[\bm{r}_{\sigma}\odot\bm{h}_{\mu}+\bm{r}_{\mu},\bm{r}_{\sigma}\odot\bm{h}_{\mu}+\bm{r}_{\mu}+\sqrt{3}|\bm{r}_{\sigma}\odot\bm{h}_{\sigma_{1}}|] & \text{if }x> \bm{r}_{\sigma}\odot\bm{h}_{\mu}+\bm{r}_{\mu}
   \end{cases}\\
   &f_{\bm{t}}(\bm{x}_{0})\sim
   \begin{cases}
   U[\bm{t}_{\mu}-\sqrt{3}|\bm{t}_{\sigma_{1}}|,\bm{t}_{\mu}] &\text{if } x\leq \bm{t}_{\mu}\\
   U[\bm{t}_{\mu},\bm{t}_{\mu}+\sqrt{3}|\bm{t}_{\sigma_{1}}|] &\text{if } x>\bm{t}_{\mu}
   \end{cases}
\end{align*}
Here, we restrict $\bm{r}_{\sigma}\odot\bm{h}_{\sigma_{1}}\geq 0,\bm{r}_{\sigma}\odot\bm{h}_{\sigma_{2}}\geq 0,\bm{t}_{\sigma_{1}}\geq 0$ and $\bm{t}_{\sigma_{2}}\geq 0$. Then the scoring function is equal to
\begin{align*}
f(\bm{h}, \bm{r},\bm{t})=&-W(f_{\bm{r}}\circ f_{\bm{h}}(\bm{x_{0}}), f_{\bm{t}}(\bm{x_{0}}))\\
=&-\|\bm{r}_{\sigma}\odot\bm{h}_{\mu}+\bm{r}_{\mu}-\bm{t}_{\mu}\|_2^2-\frac{1}{2}\||\bm{r}_{\sigma}\odot\bm{h}_{\sigma_{1}}|-|\bm{t}_{\sigma_{1}}|\|_2^2-\frac{1}{2}\||\bm{r}_{\sigma}\odot\bm{h}_{\sigma_{2}}|-|\bm{t}_{\sigma_{2}}|\|_2^2\\
&-\sqrt{\frac{3}{4}}(\bm{r}_{\sigma}\odot\bm{h}_{\mu}+\bm{r}_{\mu}-\bm{t}_{\mu})^T(|\bm{r}_{\sigma}\odot\bm{h}_{\sigma_{2}}|-|\bm{r}_{\sigma}\odot\bm{h}_{\sigma_{1}}|+|\bm{t}_{\sigma_{1}}|-|\bm{t}_{\sigma_{2}}|)
\end{align*}
\end{proof}
\begin{proposition}
If $\bm{x}_0\sim \mathcal{N}(\bm{0},\bm{I})$, the invertible functions are  Eq.(\ref{equation:7}) and similarity metric is Eq.(\ref{equation:9}), then the scoring function is 
\begin{align*}
f(\bm{h}, \bm{r},\bm{t})=&-\|\bm{r}_{\sigma}\odot\bm{h}_{\mu}+\bm{r}_{\mu}-\bm{t}_{\mu}\|_2^2-\frac{1}{2}\||\bm{r}_{\sigma}\odot\bm{h}_{\sigma_{1}}|-|\bm{t}_{\sigma_{1}}|\|_2^2-\frac{1}{2}\||\bm{r}_{\sigma}\odot\bm{h}_{\sigma_{2}}|-|\bm{t}_{\sigma_{2}}|\|_2^2\notag\\
&-\sqrt{\frac{2}{\pi}}(\bm{r}_{\sigma}\odot\bm{h}_{\mu}+\bm{r}_{\mu}-\bm{t}_{\mu})^T(|\bm{r}_{\sigma}\odot\bm{h}_{\sigma_{2}}|-|\bm{r}_{\sigma}\odot\bm{h}_{\sigma_{1}}|+|\bm{t}_{\sigma_{1}}|-|\bm{t}_{\sigma_{2}}|)
\end{align*}
\end{proposition}
\begin{proof}
We first show that for two 1-dimensional distributions
\begin{align*}
   p(x)=
   \begin{cases}
   \mathcal{N}(\mu_{1},\sigma_{1}) & \text{if }x\leq \mu_{1}\\
   \mathcal{N}(\mu_{1},\sigma_{2}) & \text{if }x> \mu_{1}
   \end{cases},
   q(x)=
   \begin{cases}
   \mathcal{N}(\mu_{2},\sigma_{3}) & \text{if }x\leq \mu_{2}\\
   \mathcal{N}(\mu_{2},\sigma_{4}) & \text{if }x>\mu_{2}
   \end{cases}
\end{align*}
the Wasserstein distance is equal to $$W(p,q)=(\mu_{1}-\mu_{2})^2+\frac{1}{2}(\sigma_{1}-\sigma_{3})^2+\frac{1}{2}(\sigma_{2}-\sigma_{4})^2+\sqrt{\frac{2}{\pi}}(\mu_{1}-\mu_{2})(\sigma_{2}-\sigma_{1}+\sigma_{3}-\sigma_{4})$$
1-dimensional Wasserstein distance is equal to
$$W(p,q)=\int_{0}^{1}|F^{-1}(z)-G^{-1}(z)|\mathrm{d}z$$
where $F^{-1}(z)$ and $G^{-1}(z)$ are the inverse cumulative distribution function of $p(x)$ and $q(x)$, $i.e.$,
\begin{align*}
   F^{-1}(z)=
   \begin{cases}
   \mu_{1}+\sqrt{2}\sigma_{1}\erf^{-1}(2z-1) & \text{if }0\leq z\leq \frac{1}{2}\\
   \mu_{1}+\sqrt{2}\sigma_{2}\erf^{-1}(2z-1) & \text{if }\frac{1}{2}< z\leq 1
   \end{cases}\\
   G^{-1}(z)
   \begin{cases}
   \mu_{2}+\sqrt{2}\sigma_{3}\erf^{-1}(2z-1) & \text{if }0\leq z\leq \frac{1}{2}\\
   \mu_{2}+\sqrt{2}\sigma_{4}\erf^{-1}(2z-1) & \text{if }\frac{1}{2}< z\leq 1
   \end{cases}
\end{align*}
Therefore
\begin{align*}
    W(p,q)=&\int_{0}^{\frac{1}{2}}|F^{-1}(z)-G^{-1}(z)|^2\mathrm{d}z+\int_{\frac{1}{2}}^{1}|F^{-1}(z)-G^{-1}(z)|^2\mathrm{d}z\\
    =&\int_{0}^{\frac{1}{2}}((\mu_{1}-\mu_{2})+\sqrt{2}\erf^{-1}(2z-1)(\sigma_{1}-\sigma_{3}))^2\mathrm{d}z\\
    +&\int_{\frac{1}{2}}^{1}((\mu_{1}-\mu_{2})+\sqrt{2}\erf^{-1}(2z-1)(\sigma_{2}-\sigma_{4}))^2\mathrm{d}z\\
    =&\int_{0}^{1}(\mu_{1}-\mu_{2})^2\mathrm{d}z+\int_{0}^{\frac{1}{2}}2\erf^{-1}(2z-1)^2(\sigma_{1}-\sigma_{3})^2\mathrm{d}z\\
    +&\int_{0}^{\frac{1}{2}}2\sqrt{2}(\mu_{1}-\mu_{2})\erf^{-1}(2z-1)(\sigma_{1}-\sigma_{3})\mathrm{d}z+\int_{\frac{1}{2}}^{1}2\erf^{-1}(2z-1)^2(\sigma_{2}-\sigma_{4})^2\mathrm{d}z\\
    +&\int_{\frac{1}{2}}^{1}2\sqrt{2}(\mu_{1}-\mu_{2})\erf^{-1}(2z-1)(\sigma_{2}-\sigma_{4})\mathrm{d}z\\
    =&(\mu_{1}-\mu_{2})^2+(\sigma_{1}-\sigma_{3})^2\int_{-1}^{0}\erf^{-1}(z)^2\mathrm{d}z
    +\sqrt{2}(\mu_{1}-\mu_{2})(\sigma_{1}-\sigma_{3})\int_{-1}^{0}\erf^{-1}(z)\mathrm{d}z\\
    +&(\sigma_{2}-\sigma_{4})^2\int_{0}^{1}\erf^{-1}(z)^2\mathrm{d}z+\sqrt{2}(\mu_{1}-\mu_{2})(\sigma_{2}-\sigma_{4})\int_{0}^{1}\erf^{-1}(z)\mathrm{d}z\\
    =&(\mu_{1}-\mu_{2})^2+\frac{1}{2}(\sigma_{1}-\sigma_{3})^2+\frac{1}{2}(\sigma_{2}-\sigma_{4})^2+\sqrt{\frac{2}{\pi}}(\mu_{1}-\mu_{2})(\sigma_{2}-\sigma_{1}+\sigma_{3}-\sigma_{4})
\end{align*}
By Corollary \ref{corollary:1}, we have that for two distributions
\begin{align*}
   p(x)=
   \begin{cases}
   \mathcal{N}(\bm{\mu}_{1},\bm{\sigma}_{1}) & \text{if }x\leq \bm{\mu}_{1}\\
   \mathcal{N}(\bm{\mu}_{1},\bm{\sigma}_{2}) & \text{if }x> \bm{\mu}_{1}
   \end{cases},
   q(x)=
   \begin{cases}
   \mathcal{N}(\bm{\mu}_{2},\bm{\sigma}_{3}) & \text{if }x\leq \bm{\mu}_{2}\\
   \mathcal{N}(\bm{\mu}_{2},\bm{\sigma}_{4}) & \text{if }x> \bm{\mu}_{2}
   \end{cases}
\end{align*}
the Wasserstein distance is equal to
\begin{align*}
W(p(\bm{x}),q(\bm{x}))=\sum_{i=1}^{n}W(p(\bm{x}_i),q(\bm{x}_i))=&\|\bm{\mu}_{1}-\bm{\mu}_{2}\|^2+\frac{1}{2}\|\bm{\sigma_{1}}-\bm{\sigma_{3}}\|^2+\frac{1}{2}\|\bm{\sigma_{2}}-\bm{\sigma_{4}}\|^2\\+&\sqrt{\frac{2}{\pi}}(\bm{\mu}_{1}-\bm{\mu}_{2})^T(\bm{\sigma}_{2}-\bm{\sigma}_{1}+\bm{\sigma}_{3}-\bm{\sigma}_{4})
\end{align*}
For every triplet $(\bm{h}, \bm{r}, \bm{t})$, we have that
\begin{align*}
   f_{\bm{r}}\circ f_{\bm{h}}(\bm{x})=
   \begin{cases}
   \bm{r}_{\sigma}\odot\bm{h}_{\sigma_{1}}\odot\bm{x}+\bm{r}_{\sigma}\odot\bm{h}_{\mu}+\bm{r}_{\mu} & \text{if }\bm{x}\leq \bm{0}\\
   \bm{r}_{\sigma}\odot\bm{h}_{\sigma_{2}}\odot\bm{x}+\bm{r}_{\sigma}\odot\bm{h}_{\mu}+\bm{r}_{\mu} & \text{if }\bm{x}>\bm{0}
   \end{cases}, 
   f_{\bm{t}}(\bm{x})=
   \begin{cases}
   \bm{t}_{\sigma_{1}}\odot\bm{x}+\bm{t}_{\mu} & \text{if }\bm{x}\leq \bm{0}\\
   \bm{t}_{\sigma_{2}}\odot\bm{x}+\bm{t}_{\mu} & \text{if }\bm{x}>\bm{0}
   \end{cases}
\end{align*}
Since $\bm{x}_{0}\sim \mathcal{N}(\bm{0},\bm{I})$, by Eq.(\ref{equation:1}), we have that
\begin{align*}
   &f_{\bm{r}}\circ f_{\bm{h}}(\bm{x}_{0})\sim
   \begin{cases}
   \mathcal{N}(\bm{r}_{\sigma}\odot\bm{h}_{\mu}+\bm{r}_{\mu},|\bm{r}_{\sigma}\odot\bm{h}_{\sigma_{1}}|) & \text{if }x\leq \bm{r}_{\sigma}\odot\bm{h}_{\mu}+\bm{r}_{\mu}\\
   \mathcal{N}(\bm{r}_{\sigma}\odot\bm{h}_{\mu}+\bm{r}_{\mu},|\bm{r}_{\sigma}\odot\bm{h}_{\sigma_{2}}|) & \text{if }x> \bm{r}_{\sigma}\odot\bm{h}_{\mu}+\bm{r}_{\mu}
   \end{cases}\\
   &f_{\bm{t}}(\bm{x}_{0})\sim
   \begin{cases}
   \mathcal{N}(\bm{t}_{\mu},|\bm{t}_{\sigma_{1}}|) & \text{if }x\leq \bm{t}_{\mu}\\
   \mathcal{N}(\bm{t}_{\mu},|\bm{t}_{\sigma_{2}}|) &\text{if } x>\bm{t}_{\mu}
   \end{cases}
\end{align*}
Here, we restrict $\bm{r}_{\sigma}\odot\bm{h}_{\sigma_{1}}\geq 0,\bm{r}_{\sigma}\odot\bm{h}_{\sigma_{2}}\geq 0,\bm{t}_{\sigma_{1}}\geq 0$ and $\bm{t}_{\sigma_{2}}\geq 0$. Then the scoring function is equal to
\begin{align*}
f(\bm{h}, \bm{r},\bm{t})=&-W(f_{\bm{r}}\circ f_{\bm{h}}(\bm{x_{0}}), f_{\bm{t}}(\bm{x_{0}}))\\
=&-\|\bm{r}_{\sigma}\odot\bm{h}_{\mu}+\bm{r}_{\mu}-\bm{t}_{\mu}\|_2^2-\frac{1}{2}\||\bm{r}_{\sigma}\odot\bm{h}_{\sigma_{1}}|-|\bm{t}_{\sigma_{1}}|\|_2^2-\frac{1}{2}\||\bm{r}_{\sigma}\odot\bm{h}_{\sigma_{2}}|-|\bm{t}_{\sigma_{2}}|\|_2^2\\
&-\sqrt{\frac{2}{\pi}}(\bm{r}_{\sigma}\odot\bm{h}_{\mu}+\bm{r}_{\mu}-\bm{t}_{\mu})^T(|\bm{r}_{\sigma}\odot\bm{h}_{\sigma_{2}}|-|\bm{r}_{\sigma}\odot\bm{h}_{\sigma_{1}}|+|\bm{t}_{\sigma_{1}}|-|\bm{t}_{\sigma_{2}}|)
\end{align*}
\end{proof}
\begin{proposition}
Let $k>0$, if $\bm{x}_{k}\sim U[-\frac{\sqrt{3}}{k},\frac{\sqrt{3}}{k}]^n$ or $\bm{x}_{k}\sim \mathcal{N}(\bm{0},\frac{\bm{I}}{k^2})$, the invertible functions are  Eq.(\ref{equation:7}) and similarity metric is Eq.(\ref{equation:9}), denote the scoring function as $f_k(\bm{h}, \bm{r},\bm{t})$, then $\bm{x}_k$ tends to a Dirac delta distribution as $k$ tends to infinity and
\begin{align*}
\lim_{k\rightarrow \infty}f_k(\bm{h}, \bm{r},\bm{t})=&\lim_{k\rightarrow \infty}-\|\bm{r}_{\sigma}\odot\bm{h}_{\mu}+\bm{r}_{\mu}-\bm{t}_{\mu}\|_2^2-\frac{1}{k^2}\||\bm{r}_{\sigma}\odot\bm{h}_{\sigma}|-|\bm{t}_{\sigma}|\|_2^2\notag\\
=&-\|\bm{r}_{\sigma}\odot\bm{h}_{\mu}+\bm{r}_{\mu}-\bm{t}_{\mu}\|_2^2
\end{align*}
\end{proposition}
\begin{proof}
Since $k\bm{x}_{k}\sim U[-\sqrt{3},\sqrt{3}]^n$, by Proposition \ref{proposition:2}, we have that
\begin{align*}
f_k(\bm{h}, \bm{r},\bm{t})=&-\|\bm{r}_{\sigma}\odot\bm{h}_{\mu}+\bm{r}_{\mu}-\bm{t}_{\mu}\|_2^2-\||\bm{r}_{\sigma}\odot\frac{1}{k}\bm{h}_{\sigma}|-|\frac{1}{k}\bm{t}_{\sigma}|\|_2^2\\
=&-\|\bm{r}_{\sigma}\odot\bm{h}_{\mu}+\bm{r}_{\mu}-\bm{t}_{\mu}\|_2^2-\frac{1}{k^2}\||\bm{r}_{\sigma}\odot\bm{h}_{\sigma}|-|\bm{t}_{\sigma}|\|_2^2
\end{align*}
then 
\begin{align*}
\lim_{k\rightarrow \infty}f_k(\bm{h}, \bm{r},\bm{t})
=-\|\bm{r}_{\sigma}\odot\bm{h}_{\mu}+\bm{r}_{\mu}-\bm{t}_{\mu}\|_2^2
\end{align*}
\end{proof}
\begin{proposition}
Scoring functions Eq.(\ref{equation:10}), Eq.(\ref{equation:11}) and Eq.(\ref{equation:12}) can learn symmetry, antisymmetry, inverse and composition rules.
\end{proposition}
\begin{proof}
Since Eq.(\ref{equation:11}) or Eq.(\ref{equation:12}) reduces to Eq.(\ref{equation:10}) if $\bm{h}_{\sigma_{1}}=\bm{h}_{\sigma_{2}}$ and $\bm{t}_{\sigma_{1}}=\bm{t}_{\sigma_{2}}$, we only proof for Eq.(\ref{equation:10}):
$$f(\bm{h}, \bm{r}, \bm{t})=-\|\bm{r}_{\sigma}\odot\bm{h}_{\mu}+\bm{r}_{\mu}-\bm{t}_{\mu}\|_2^2-\||\bm{r}_{\sigma}\odot\bm{h}_{\sigma}|-|\bm{t}_{\sigma}|\|_2^2$$
\textbf{Symmetry Rules}: If $(h,r,t)\in \mathcal{F}\wedge(t,r,h)\in \mathcal{F}$ hold, we have
\begin{align*}
  &\bm{r}_{\sigma}\odot\bm{h}_{\mu}+\bm{r}_{\mu}=\bm{t}_{\mu},
|\bm{r}_{\sigma}\odot\bm{h}_{\sigma}|=|\bm{t}_{\sigma}|\\
  &\bm{r}_{\sigma}\odot\bm{t}_{\mu}+\bm{r}_{\mu}=\bm{h}_{\mu},
|\bm{r}_{\sigma}\odot\bm{t}_{\sigma}|=|\bm{h}_{\sigma}|
\end{align*}
then we have
\begin{align*}
  &\bm{r}_{\sigma}\odot\bm{r}_{\sigma}=\bm{1}\\
  &\bm{r}_{\sigma}\odot\bm{r}_{\mu}+\bm{r}_{\mu}=0
\end{align*}
\textbf{Antisymmetry Rules}:  If $(h,r,t)\in \mathcal{F}\wedge\neg(t,r,h)\in \mathcal{F}$ hold, we have
\begin{align*}
  &\bm{r}_{\sigma}\odot\bm{h}_{\mu}+\bm{r}_{\mu}=\bm{t}_{\mu},
|\bm{r}_{\sigma}\odot\bm{h}_{\sigma}|=|\bm{t}_{\sigma}|\\
  &\bm{r}_{\sigma}\odot\bm{t}_{\mu}+\bm{r}_{\mu}\neq\bm{h}_{\mu} \text{, or }
|\bm{r}_{\sigma}\odot\bm{t}_{\sigma}|\neq|\bm{h}_{\sigma}|
\end{align*}
then we have
\begin{align*}
  &\bm{r}_{\sigma}\odot\bm{r}_{\sigma}\neq\bm{1}\text{, or }\\
  &\bm{r}_{\sigma}\odot\bm{r}_{\mu}+\bm{r}_{\mu}\neq0
\end{align*}
\textbf{Inverse Rules}: If $(h,r_1,t)\in \mathcal{F}\wedge(t,r_2,h)\in \mathcal{F}$ hold, we have
\begin{align*}
  &\bm{r}_{1\sigma}\odot\bm{h}_{\mu}+\bm{r}_{1\mu}=\bm{t}_{\mu},
|\bm{r}_{1\sigma}\odot\bm{h}_{\sigma}|=|\bm{t}_{\sigma}|\\
  &\bm{r}_{2\sigma}\odot\bm{t}_{\mu}+\bm{r}_{2\mu}=\bm{h}_{\mu},
|\bm{r}_{2\sigma}\odot\bm{t}_{\sigma}|=|\bm{h}_{\sigma}|
\end{align*}
then we have
\begin{align*}
  &\bm{r}_{2\sigma}\odot\bm{r}_{1\sigma}=\bm{1}\\
  &\bm{r}_{2\sigma}\odot\bm{r}_{1\mu}+\bm{r}_{2\mu}=0
\end{align*}
\textbf{Composition Rules}: If $(h,r_{1},\widetilde{t})\in \mathcal{F}\wedge (\widetilde{t},r_{2},t)\in \mathcal{F} \wedge (h,r_{3},t)\in \mathcal{F}$ hold, we have
\begin{align*}
  &\bm{r}_{1\sigma}\odot\bm{h}_{\mu}+\bm{r}_{1\mu}=\bm{\widetilde{t}}_{\mu},
|\bm{r}_{1\sigma}\odot\bm{h}_{\sigma}|=|\bm{\widetilde{t}}_{\sigma}|\\
  &\bm{r}_{2\sigma}\odot\bm{\widetilde{t}}_{\mu}+\bm{r}_{2\mu}=\bm{t}_{\mu},
|\bm{r}_{2\sigma}\odot\bm{\widetilde{t}}_{\sigma}|=|\bm{t}_{\sigma}|\\
  &\bm{r}_{3\sigma}\odot\bm{h}_{\mu}+\bm{r}_{3\mu}=\bm{t}_{\mu},
|\bm{r}_{3\sigma}\odot\bm{h}_{\sigma}|=|\bm{t}_{\sigma}|
\end{align*}
then we have
\begin{align*}
  &\bm{r}_{1\sigma}\odot\bm{r}_{2\sigma}=\bm{r}_{3\sigma}\\
  &\bm{r}_{2\sigma}\odot\bm{r}_{1\mu}+\bm{r}_{2\mu}=\bm{r}_{3\mu}
\end{align*}
\end{proof}

\section{Experimental Details}
\label{appendix:4}
\paragraph{Datasets}
We evaluate our model on three popular knowledge graph completion datasets, WN18RR \cite{dettmers2018convolutional}, FB15k-237 \cite{toutanova2015representing} and YAGO3-10 \cite{dettmers2018convolutional}. WN18RR is a subset of WN18, with inverse relations removed. WN18 is extracted from WordNet, a database containing lexical relations between words. FB15k-237 is a subset of FB15k, with inverse relations removed. FB15k is extracted from Freebase, a large database of real world facts. YAGO3-10 is a subset of YAGO3 that only contains entities with at least 10 relations. The statistics of the datasets are shown in Table \ref{table:4}.
\begin{table}[ht]
\caption{The statistics of the datasets.}
\label{table:4}
\vskip 0.15in
\begin{center}
\begin{small}
\begin{tabular}{llllll}
\toprule
Dataset &\#entity &\#relation &\#train &\#valid &\#test\\
\midrule
WN18RR    &40,943 &11 &86,835 &3,034 &3,134\\
FB15k-237 &14,541 &237 &272,115 &17,535 &20,466\\
YGAO3-10  &123,188 &37 &1,079,040 &5,000 &5,000\\
\bottomrule
\end{tabular}
\end{small}
\end{center}
\vskip -0.1in
\end{table}

\paragraph{Evaluation Metrics}
MR=$\frac{1}{N}\sum_{i=1}^{N}\rank_{i}$, where $\rank_{i}$ is the rank of $i$th triplet in the test set and $N$ is the number of the triplets. Lower MR indicates better performance.

MRR=$\frac{1}{N}\sum_{i=1}^{N}\frac{1}{\rank_{i}}$. Higher MRR indicates better performance.

$\text{Hits@N}=\frac{1}{N}\sum_{i=1}^{N}\mathbb{I}({\rank_{i}\leq N}$), where $\mathbb{I}(\cdot)$ is the indicator function. $\text{Hits@N}$ is the ratio of the ranks that no more than $N$, Higher $\text{Hits@N}$ indicates better performance.
\paragraph{Hyper-parameters}
We used Adam \cite{kingma2014adam} with exponential decay as the optimizer. We set the embedding dimension to 1024 for all models. We search the learning rate in $\{0.0005, 0.001, 0.003, 0.005, 0.01\}$, decay rate in $\{0.9, 0.93, 0.95, 1.0\}$, batch size in $\{128, 256, 512, 1024\}$, margin in $\{0, 1, 2, 3,4,5,6,7,8,9,10,11,12,13,14,15\}$. We first search the best margin, then search other hyperparameters. Denote Eq.(\ref{equation:15}) as NFE-3. We further show the result of NFE-3 in Table \ref{table:5}. See Table \ref{table:6}, Table \ref{table:7} and Table \ref{table:8} for the best hyper-parameters we searched.

\begin{table}[ht]
\caption{The results of NFE-3 on WN18RR, FB15k-237 and YAGO3-10 datasets.}
\label{table:5}
\vskip 0.15in
\begin{center}
\begin{small}
\begin{tabular}{llll}
\toprule
Dataset &MRR    &H@1    &H@10\\
\midrule
WN18RR    &0.440  &0.400  &0.521\\
FB15k-237 &0.340   &0.246  &0.530\\
YGAO3-10  &0.550   &0.470  &0.695\\
\bottomrule
\end{tabular}
\end{small}
\end{center}
\vskip -0.1in
\end{table}

\begin{table}[ht]
\caption{The hyper-parameters of NFE-1.}
\label{table:6}
\vskip 0.15in
\begin{center}
\begin{small}
\begin{tabular}{lllll}
\toprule
Dataset &learning rate &decay rate &batch size &margin\\
\midrule
WN18RR    &0.005    &0.9    &128 &1\\
FB15k-237 &0.001   &0.93   &1024 &2 \\
YGAO3-10  &0.001   &0.95  &1024 &8 \\
\bottomrule
\end{tabular}
\end{small}
\end{center}
\vskip -0.1in
\end{table}

\begin{table}[ht]
\caption{The hyper-parameters of NFE-2.}
\label{table:7}
\vskip 0.15in
\begin{center}
\begin{small}
\begin{tabular}{lllll}
\toprule
Dataset &learning rate &decay rate &batch size &margin\\
\midrule
WN18RR    &0.003    &0.93   &128 &0\\
FB15k-237 &0.001    &0.9   &256 &1\\
YGAO3-10  &0.001   &0.95  &1024 &7\\
\bottomrule
\end{tabular}
\end{small}
\end{center}
\vskip -0.1in
\end{table}

\begin{table}[ht]
\caption{The hyper-parameters of NFE-3.}
\label{table:8}
\vskip 0.15in
\begin{center}
\begin{small}
\begin{tabular}{lllll}
\toprule
Dataset &learning rate &decay rate &batch size &margin\\
\midrule
WN18RR    &0.01    &0.93    &512 &14\\
FB15k-237 &0.001  &0.9   &256  &2\\
YGAO3-10  &0.001   &0.93 &1024 & 11\\
\bottomrule
\end{tabular}
\end{small}
\end{center}
\vskip -0.1in
\end{table}

\paragraph{Logical Rules}
We have proved that NFE-1 is able to learn logical rules in Proposition \ref{proposition:6}. We compute the relevant statistics to verify whether NFE-1 can learn symmetry, antisymmetry, inverse and composition rules.

\textbf{Symmetry Rules}:
We investigate the embeddings of relations from a 1024-dimensional NFE-1 trained on WN18RR dataset.
Figure \ref{figure:3} shows the histogram of the statistic $(|\bm{r}_{\sigma}\odot\bm{r}_{\sigma}-\bm{1}|,|\bm{r}_{\sigma}\odot\bm{r}_{\mu}+\bm{r}_{\mu})|$ from a symmetry relation $\emph{similar\_to}$. We can
find that most of the values are close to 0. This shows that NFE-1 can learn symmetry rules.

\textbf{Antisymmetry Rules}:
We investigate the embeddings of relations from a 1024-dimensional NFE-1 trained on WN18RR dataset.
Figure \ref{figure:4} shows the histogram of the statistic $(|\bm{r}_{\sigma}\odot\bm{r}_{\sigma}-\bm{1}|,|\bm{r}_{\sigma}\odot\bm{r}_{\mu}+\bm{r}_{\mu}|)$ from a antisymmetry relation $\emph{\_instance\_hypernym}$. We can
find that many of the values are not close to 0. This shows that NFE-1 can learn antisymmetry rules.

\textbf{Inverse Rules}:
We investigate the embeddings of relations from a 1024-dimensional NFE-1 trained on WN18RR dataset.
Figure \ref{figure:5} shows the histogram of the statistic $(|\bm{r}_{2\sigma}\odot\bm{r}_{1\sigma}-\bm{1}|,|\bm{r}_{2\sigma}\odot\bm{r}_{1\mu}+\bm{r}_{2\mu}|)$ from a relation $\emph{similar\_to}$ and its inverse relation $\emph{inverse\_similar\_to}$. We can
find that most of the values are close to 0. This shows that NFE-1 can learn inverse rules.

\textbf{Composition Rules}:
We investigate the embeddings of relations from a 1024-dimensional NFE-1 trained on FB15k-237 dataset.
Figure \ref{figure:6} shows the histogram of the statistic $(|\bm{r}_{1\sigma}\odot\bm{r}_{2\sigma}-\bm{r}_{3\sigma}|,|\bm{r}_{2\sigma}\odot\bm{r}_{1\mu}+\bm{r}_{2\mu}-\bm{r}_{3\mu}|)$ of a relation $\emph{/award/award\_nominee/award\_nominations./award/award\_nomination/nominated\_for}$, a relation $\emph{/award/award\_category/winners./award/award\_honor/award\_winner}$ and their composition relation
$\emph{/award/award\_category/nominees./award/award\_nomination/nominated\_for}$. We can
find that most of the values are almost 0. This shows that NFE-1 can learn composition rules.

\begin{figure}[t]
\begin{center}
\centerline{\includegraphics[width=0.8\textwidth]{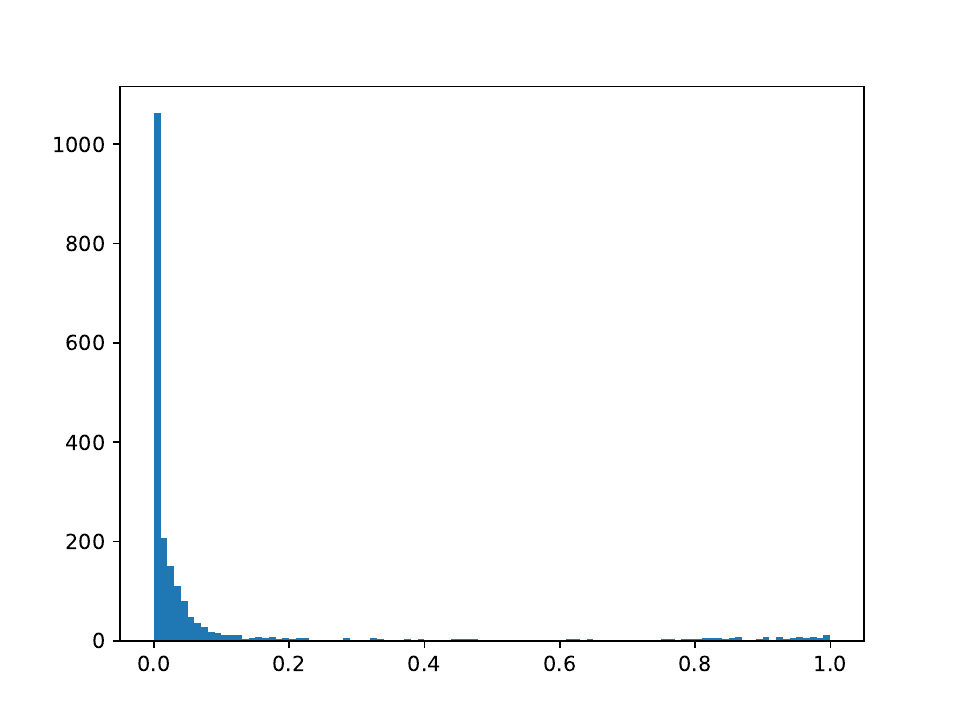}}
\caption{The histogram of the statistic $(|\bm{r}_{\sigma}\odot\bm{r}_{\sigma}-\bm{1}|,|\bm{r}_{\sigma}\odot\bm{r}_{\mu}+\bm{r}_{\mu}|)$.}
\label{figure:3}
\end{center}
\end{figure}

\begin{figure}[t]
\begin{center}
\centerline{\includegraphics[width=0.8\textwidth]{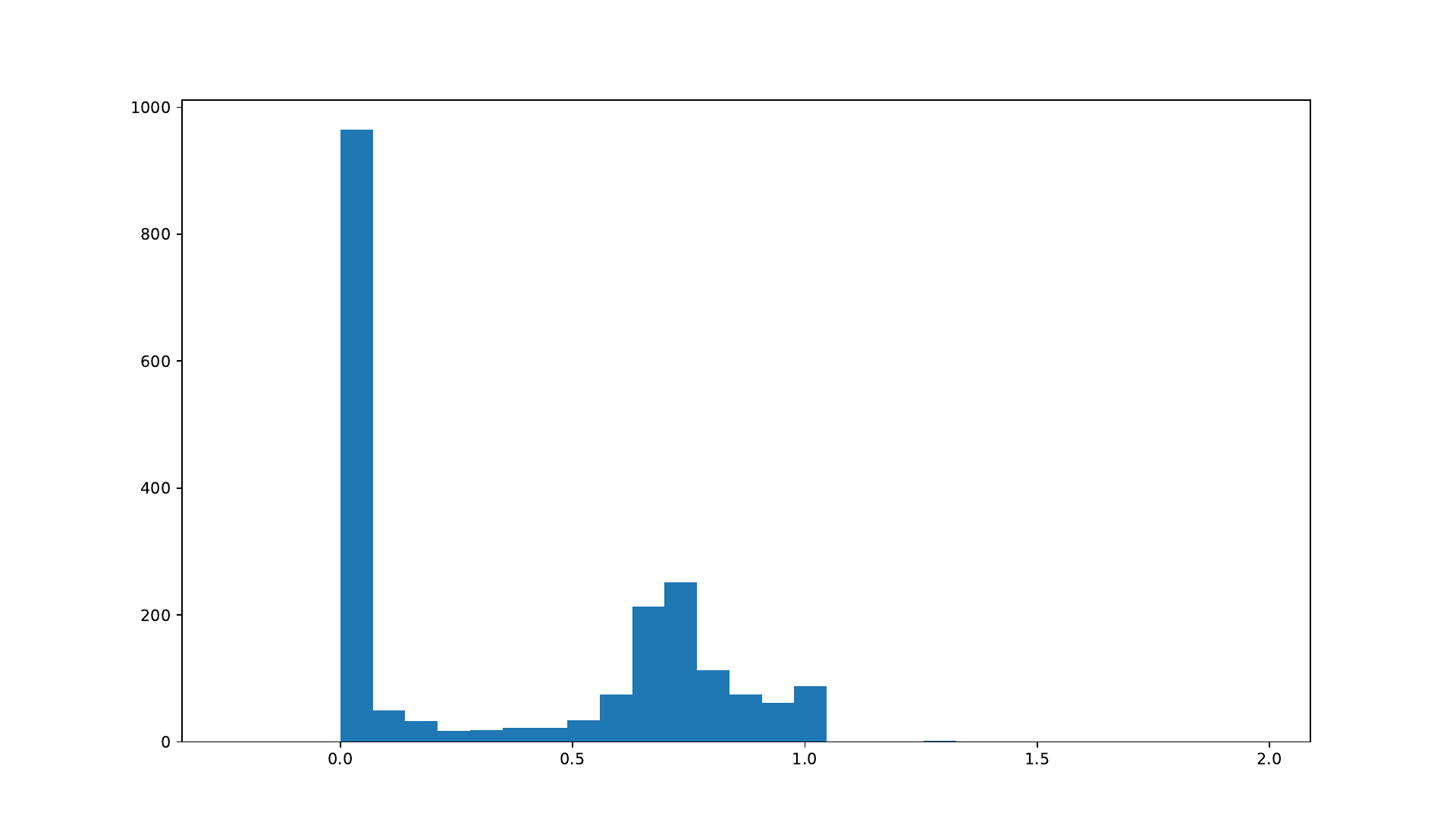}}
\caption{The histogram of the statistic $(|\bm{r}_{\sigma}\odot\bm{r}_{\sigma}-\bm{1}|,|\bm{r}_{\sigma}\odot\bm{r}_{\mu}+\bm{r}_{\mu}|)$.}
\label{figure:4}
\end{center}
\end{figure}

\begin{figure}[t]
\begin{center}
\centerline{\includegraphics[width=0.8\textwidth]{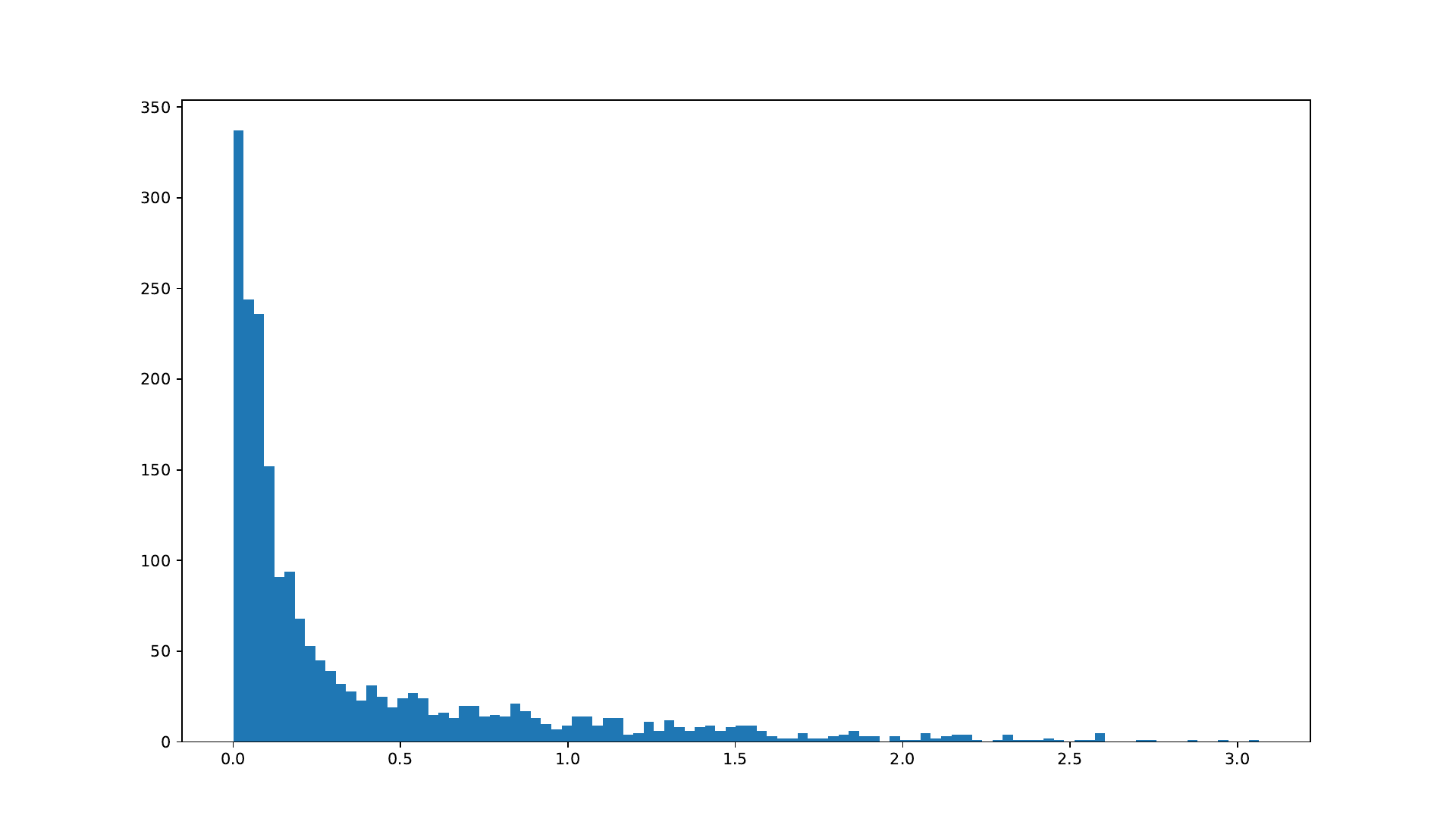}}
\caption{The histogram of the statistic $(|\bm{r}_{2\sigma}\odot\bm{r}_{1\sigma}-\bm{1}|,|\bm{r}_{2\sigma}\odot\bm{r}_{1\mu}+\bm{r}_{2\mu}|)$.}
\label{figure:5}
\end{center}
\end{figure}

\begin{figure}[t]
\begin{center}
\centerline{\includegraphics[width=0.8\textwidth]{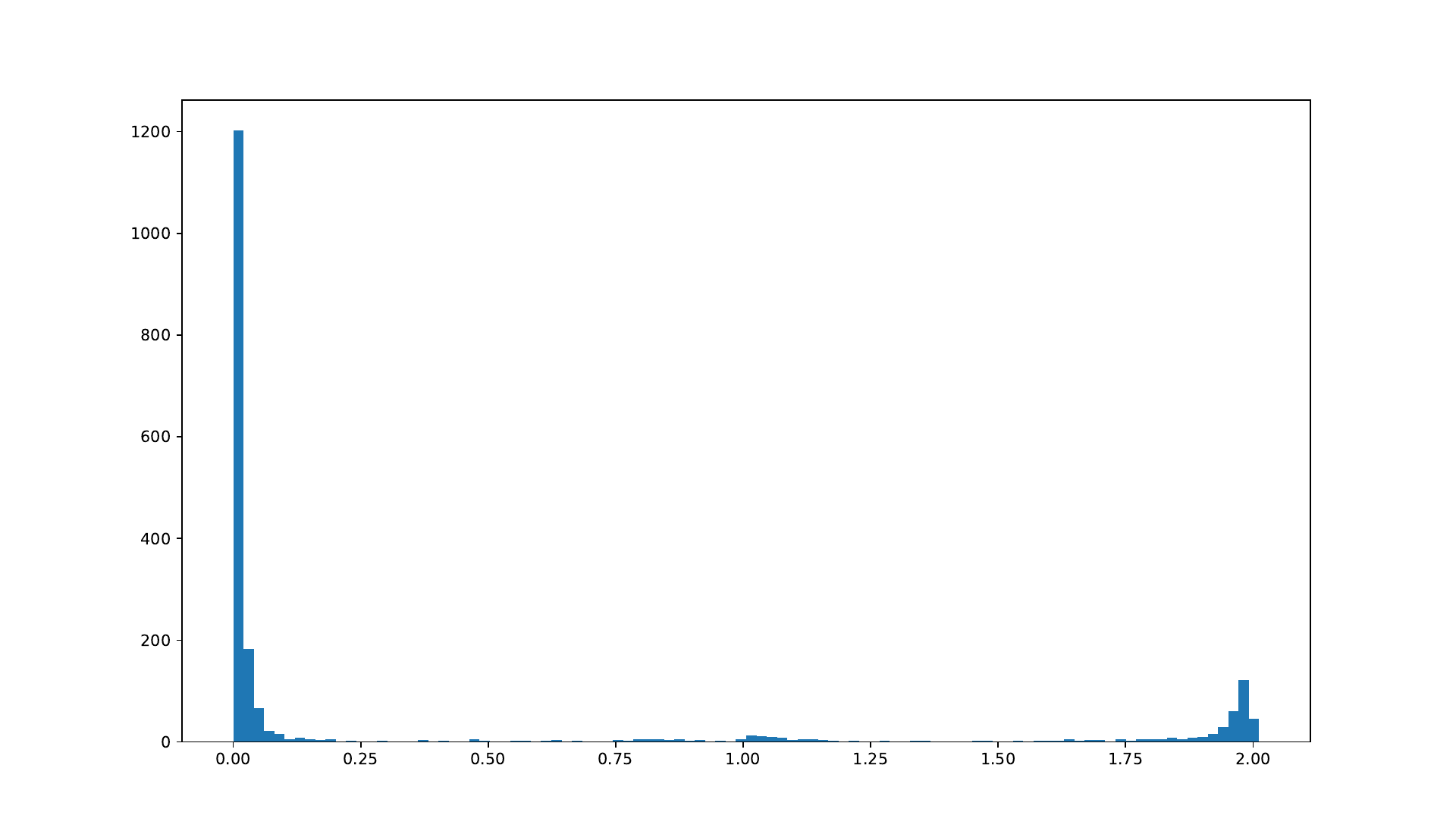}}
\caption{The histogram of the statistic $(|\bm{r}_{1\sigma}\odot\bm{r}_{2\sigma}-\bm{r}_{3\sigma}|,|\bm{r}_{2\sigma}\odot\bm{r}_{1\mu}+\bm{r}_{2\mu}-\bm{r}_{3\mu}|)$.}
\label{figure:6}
\end{center}
\end{figure}


\end{document}